\documentclass{article} %
\usepackage{iclr2026_conference,times}

\usepackage{amsmath,amsfonts,bm}

\def\eqref#1{equation~\ref{#1}}

\def\1{\bm{1}}

\DeclareMathAlphabet{\mathsfit}{\encodingdefault}{\sfdefault}{m}{sl}
\SetMathAlphabet{\mathsfit}{bold}{\encodingdefault}{\sfdefault}{bx}{n}

\newcommand{\E}{\mathbb{E}}

\newcommand{\R}{\mathbb{R}}

\newcommand{\KL}{D_{\mathrm{KL}}}

\DeclareMathOperator*{\argmin}{arg\,min}

\usepackage{amsmath,amssymb,mathtools}
\usepackage{amsthm}
\usepackage{booktabs}
\usepackage{enumitem}
\usepackage{xcolor}
\usepackage{tikz}
\usepackage{pgfplots}\pgfplotsset{compat=1.18}
\usepgfplotslibrary{groupplots,fillbetween}
\usetikzlibrary{positioning,calc}
\usepackage[noend]{algpseudocode}
\usepackage{algorithm}
\usepackage[most]{tcolorbox}
\usepackage{titletoc}
\usepackage{hyperref}
\usepackage{url}
\usepackage{fontawesome5}
\usepackage{aliascnt}
\usepackage[capitalize,noabbrev]{cleveref}
\hypersetup{hypertexnames=false} %

\iclrfinalcopy %

\definecolor{cteal}{RGB}{42,157,143}
\definecolor{cyellow}{RGB}{223,178,88}
\definecolor{cforce}{RGB}{214,73,51}

\definecolor{clink}{RGB}{36,133,122} %
\hypersetup{colorlinks=true, citecolor=clink, linkcolor=clink, urlcolor=clink}

\newcommand{\gauss}[3]{\mathcal{N}\!\left(#1;\,#2,\,#3\right)}
\newcommand{\kl}[2]{\KL\!\left(#1\,\|\,#2\right)}
\newcommand{\dt}{\mathrm{d}}
\newcommand{\piplus}{\pi^{\star}}
\newcommand{\obs}{s}                 %
\newcommand{\score}{\mathsf{s}}      %
\newcommand{\vel}{v}
\newcommand{\cs}{\kappa}              %
\newcommand{\velref}{\vel_{\mathrm{ref}}}
\newcommand{\scoreref}{\score_{\mathrm{ref}}}
\newcommand{\pref}{p_{\mathrm{ref}}}

\newcommand{\tauou}{\tau_{\mathrm{ou}}}
\newcommand{\Eint}[1]{\E_{X_1 \sim #1,\,\varepsilon}}       %
\newcommand{\Eintt}[1]{\E_{t,\,X_1 \sim #1,\,\varepsilon}}  %
\newcommand{\method}{FFM}

\theoremstyle{plain}
\newtheorem{theorem}{Theorem}
\newaliascnt{proposition}{theorem}
\newtheorem{proposition}[proposition]{Proposition}
\aliascntresetthe{proposition}
\newaliascnt{lemma}{theorem}

\aliascntresetthe{lemma}
\newaliascnt{corollary}{theorem}
\newtheorem{corollary}[corollary]{Corollary}
\aliascntresetthe{corollary}
\newtheorem*{theorem*}{Theorem}
\theoremstyle{remark}
\newaliascnt{remark}{theorem}

\aliascntresetthe{remark}
\newaliascnt{assumption}{theorem}

\aliascntresetthe{assumption}
\crefname{proposition}{Proposition}{Propositions}
\crefname{lemma}{Lemma}{Lemmas}
\crefname{corollary}{Corollary}{Corollaries}
\crefname{remark}{Remark}{Remarks}
\crefname{assumption}{Assumption}{Assumptions}
\newcommand{\thmref}[1]{Theorem~\ref{#1}}
\newcommand{\propref}[1]{Proposition~\ref{#1}}

\newcommand{\corref}[1]{Corollary~\ref{#1}}

\newtcolorbox{tealbox}{
  enhanced, breakable,
  colframe=cteal!85!black, colback=cteal!8,
  boxrule=0.9pt, arc=6pt, left=7pt, right=7pt, top=6pt, bottom=6pt,
  before skip=6pt, after skip=6pt}

\title{Reinforcement Learning \\ via Fixed-point Flow Matching}

\author{Denis Blessing\thanks{Corresponding author: \texttt{denis.blessing@kit.edu}} \quad Onur Celik \quad Gerhard Neumann \\
Karlsruhe Institute of Technology (KIT) \\
Autonomous Learning Robots Lab (ALR)}

\begin{document}

\maketitle
\fancyhead[L]{Preprint}   %
\begin{abstract}
In maximum entropy and behavior-constrained reinforcement learning (RL), the optimal policy is given by a Boltzmann distribution. Generating samples from this distribution can be challenging, as it is often complex and multimodal. While flow policies trained with flow matching provide a powerful framework for representing such distributions, they traditionally rely on access to training data, which is generally unavailable in online RL. To address this limitation, we propose \textit{Fixed-point Flow Matching (FFM)}, a data-free procedure for training and fine-tuning flow policies in RL. FFM leverages the target flow identity to formulate the actor update as a fixed-point iteration, utilizing a flow-matching-style regression objective where the critic's action gradients serve as regression targets. We show that the optimal policy is a fixed point of this iteration. Through extensive experiments, we demonstrate that FFM effectively fine-tunes pretrained flow policies in the offline-to-online setting and enables efficient and accurate training of flow policies from scratch in online RL. These results establish fixed-point flow matching as a principled approach for learning expressive flow-based policies without requiring training data.
\end{abstract}

\begin{center}
  \small
  \faGlobe\;\href{https://denisbless.github.io/FFM_website/}{Project page}
  \qquad
  \faGithub\;\href{https://github.com/DenisBless/FFM_RL}{Code}
\end{center}

\section{Introduction}
\label{sec:intro}

Across a range of reinforcement learning (RL) formulations, the improved policy is available in closed form as a Boltzmann distribution of the current critic. In maximum entropy RL, $\piplus(x \mid \obs) \propto \exp(Q(\obs,x)/\lambda)$ is optimal for the policy improvement step of soft policy iteration \citep{ziebart2008maximum,haarnoja2017sql,haarnoja2018sac}; in KL-regularized and behavior-constrained RL, which includes offline-to-online fine-tuning, the same expression appears tilted by a reference policy \citep{peters2010relative,abdolmaleki2018mpo,peng2019awr,nair2020awac}. This target distribution can be complex and multimodal whenever several distinct action choices are viable, but the classical actor--critic recipe projects it onto a Gaussian due to the simplicity of the resulting optimization problem and the tractability of the policy's log-likelihood \citep{haarnoja2018sac}. Flow and diffusion models \citep{lipman2023flow,liu2023rectified,albergo2023building,ho2020ddpm,song2021sde} remove this representational restriction and have become the policy class of choice wherever a dataset is available \citep{chi2023diffusionpolicy,reuss2023goal,janner2022diffuser,park2025fql}. Yet, the stable training objectives that are used to train these models are not directly usable for online RL due to a lack of samples from $\piplus$.
Existing methods for training flow and diffusion models in the on-policy RL setting often require backpropagating the critic through the integrator \citep{wang2023diffusionql,ding2024consistency,wang2024dacer,zhang2025sacflow, celik2025dime, trudi2025}, whose memory cost grows with every integration step and quickly becomes prohibitive, or rely on importance-weighted regression \citep{ding2024qvpo,ma2025dpmd,dong2025maxentdp}, whose exponential critic weights have high variance and lead to unstable training updates. Moreover, methods for fine-tuning flow models typically use stochastic optimal control, which requires converting the flow into a diffusion model in order to use the adjoint method \citep{domingoenrich2025adjoint,li2026qam,dong2026trust}. The adjoint method stores each sampled trajectory and integrates an adjoint state backward along it, leading to complex and expensive update schemes.

\begin{figure}[t]
\centering
\resizebox{\linewidth}{!}{\input{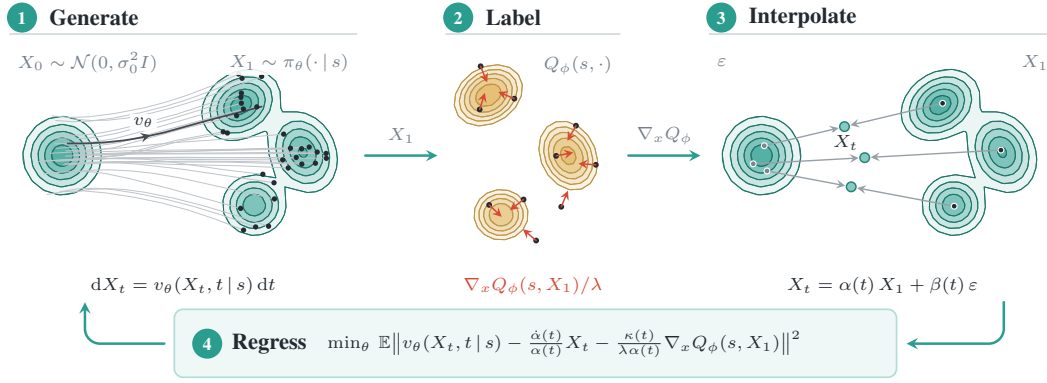}}
\caption{\textbf{The \method{} policy-improvement step.} \method{} repeats:
\textbf{(1)} Clean actions $X_1$ are generated from the current policy by integrating \cref{eq:ode} (forward only, no gradients); the grey curves are sampling trajectories.
\textbf{(2)} Critic action gradients are required for the \method{} regression target of \cref{eq:fp-tfm}.
\textbf{(3)} Fresh noise $\varepsilon$ and the clean action $X_1$ are used to sample a noised action $X_t$ via \cref{eq:interpolant}.
\textbf{(4)} A flow-matching-style regression pulls $\vel_\theta$ toward the \method{} target. The optimal policy $\piplus \propto e^{Q/\lambda}$ is a fixed point of this loop.}
\label{fig:method}
\end{figure}

In this work we ask whether a flow policy can be trained against a Boltzmann target with the same simple regression that makes flow matching stable, but without samples from the target. Our starting point is the \emph{target flow identity}: for a Boltzmann target, the velocity field that generates the noised target is, up to a rescaling of the noisy action, a conditional average of the critic's action gradients $\nabla_x Q(\obs,x)$ at clean actions. It replaces the data-derived regression target of flow matching by a critic-derived one, and the resulting objective, \emph{target flow matching}, has the same minimizer as flow matching while reading its targets off the critic instead of a dataset. Since the objective still averages over samples of the target, we evaluate it with samples of the current policy and iterate. The resulting \emph{fixed-point flow matching} (\method{}) update, illustrated in \cref{fig:method}, generates actions with the current flow, labels them with critic gradients, noises them, and regresses the velocity field onto the labels. We show that the optimal policy is a fixed point of this iteration. The whole procedure is conceptually simple, easy to implement and requires neither backpropagation through the integrator nor importance weights or conversion into a diffusion model.

\paragraph{Contributions.}
\begin{itemize}[leftmargin=*,itemsep=2pt,topsep=2pt]
  \item \textbf{The target flow identity and target flow matching} (\cref{sec:tfm}). We derive an identity that expresses the velocity field of a Boltzmann target as a conditional expectation of the gradient of its log-density, and turn it into target flow matching, an alternative training objective to flow matching that shares its population minimizer but regresses onto the gradient of the target rather than onto the displacement of the data.
  \item \textbf{Fixed-point flow matching} (\cref{sec:tfm}). We propose a fixed-point iteration that evaluates the target flow matching objective with samples of the current policy instead of the optimal one, and prove that the optimal policy is a fixed point of this iteration (\thmref{thm:fixedpoint}).
  \item \textbf{Extensions} (\cref{sec:online}). We further introduce several extensions for flow policies: a trust-region loss in velocity space, a loss for behavior-constrained fine-tuning that reads the reference score off the pretrained flow, and correlated exploration through temporally correlated prior draws.
  \item \textbf{Two RL instantiations} (\cref{sec:online}). We instantiate \method{} for maximum entropy RL and for offline-to-online fine-tuning. We validate both across $50$ offline-to-online tasks and $24$ online tasks (\cref{sec:experiments}), where \method{} is competitive with or better than the strongest available baselines.
\end{itemize}

\section{Related work}
\label{sec:rw}
\paragraph{Reinforcement learning with diffusion and flow policies.}
Diffusion and flow models entered RL through imitation learning and offline RL, where a dataset makes their standard training objectives applicable and the critic enters through weighting, guidance, or gradients through the sampling chain \citep{janner2022diffuser,chi2023diffusionpolicy,wang2023diffusionql,chen2023sfbc,hansenestruch2023idql,lu2023cep,kang2023edp,chen2024srpo,ding2024consistency,park2025fql}. In online RL there are no samples of the target policy, and existing methods differ in how the critic drives the actor update. One line backpropagates the critic value through the sampling chain \citep{wang2023diffusionql,ding2024consistency,wang2024dacer,zhang2025sacflow}, which is memory-intensive and unstable for many integration steps. A second line uses critic-weighted regression, interpreting the denoising loss as a variational bound or as a mirror-descent step \citep{ding2024qvpo,ma2025dpmd,dong2025maxentdp}, and a third estimates policy gradients through the denoising process, typically under a trust region \citep{ren2025dppo,mcallister2026flow,zhang2025reinflow,trudi2025,lv2026flac}. Closest to \method{} are methods that use the action gradient of the critic: DIPO \citep{yang2023dipo} clones actions improved by gradient ascent on $Q$; QSM \citep{psenka2024qsm} regresses the score onto $\nabla_x Q$ evaluated at the noisy action, a plug-in approximation of the target score identity that is exact only at the clean end; and RFM \citep{li2026reverse} estimates the exact posterior-mean target by self-normalized importance sampling, using Langevin--Stein control variates to interpolate between $Q$-weighted noise averages and $Q$-weighted averages of $\nabla_x Q$, which allows it to train flow policies as well. Maximum entropy RL with generative policies additionally requires an entropy estimate: DACER \citep{wang2024dacer} fits a Gaussian mixture, DIME \citep{celik2025dime} and MaxEntDP \citep{dong2025maxentdp} use a lower bound and numerical integration, respectively, \citet{sanokowski2025damdp} augment the MDP with the denoising steps, FLAG \citep{kim2026flag} uses importance sampling in the latent space of the flow, and AMDP \citep{amdp2025} amortizes the entropy-regularized update with adjoint matching. \method{} obtains the entropy from the divergence of the velocity field it trains, using the instantaneous change of variables.

\paragraph{Fine-tuning and offline-to-online RL.}
Reward fine-tuning of pretrained diffusion and flow models has been cast as policy gradient over the denoising chain \citep{fan2023dpok,black2024ddpo,liu2025flowgrpo} and as entropy-regularized stochastic optimal control \citep{uehara2024finetuning,domingoenrich2025adjoint}, with the KL divergence to the pretrained model as regularizer; \citet{uehara2025inference} review the corresponding inference-time methods. The same behavior-tilted target governs offline-to-online RL with flow policies. FQL \citep{park2025fql} distills the flow into a one-step policy that is trained by reparameterized critic gradients, DSRL \citep{wagenmaker2025dsrl} keeps the flow fixed and learns a policy over its latent noise, ReinFlow and FINO \citep{zhang2025reinflow,shin2026fino} inject noise into the flow to obtain tractable likelihoods and exploration, and Q-Flow \citep{doo2026qflow} propagates terminal values to intermediate states of the flow. QAM \citep{li2026qam} imports adjoint matching: the pretrained flow is converted into a memoryless SDE and critic gradients are transported backward along sampled trajectories by the lean adjoint ODE, and TRQAM \citep{dong2026trust} stabilizes it with a path-space trust region. Concurrent work regresses velocity corrections onto critic gradients \citep{wang2026qvgm,shi2026flowdpg,ziakas2026fmq} or steers a fixed flow policy at test time \citep{zhou2026qgf,ruan2026qpilots}. \method{} fine-tunes the pretrained flow directly in velocity space, requires neither SDE conversion nor adjoint simulation, and its behavior-tilted target has the KL-regularized optimum as a fixed point (\cref{app:o2o}).

Beyond the RL literature, \method{} shares the fixed-point structure of the matching objectives used to train diffusion and flow models from unnormalized densities, and of adjoint matching; \cref{app:extended-rw} discusses these methods, together with further critic-driven policy updates.

\section{Preliminaries}
\label{sec:prelim}

\paragraph{Reinforcement learning.}
We consider a Markov decision process with states $\obs \in \mathcal{S}$, continuous actions $x \in \mathcal{A} \subseteq \R^{d}$, transitions $P(\obs' \mid \obs, x)$, reward $r(\obs,x)$, and discount $\gamma \in [0,1)$.
The state-action value function of the current policy $\pi$ is the expected discounted return $Q^{\pi}(\obs,x) = \E_{\pi}\big[\sum_{k\ge 0}\gamma^{k} r(\obs_k,x_k) \mid \obs_0 = \obs,\, x_0 = x\big]$ when taking action $x$ in state $\obs$ and following $\pi$ thereafter.
Maximizing this return alone drives the policy towards a deterministic optimum, which often suffers from premature convergence. The two settings we consider therefore both regularize the policy improvement step, which leads to closed-form solutions that are known only up to an intractable normalization constant.

In \emph{behavior-constrained} RL \citep{peters2010relative,abdolmaleki2018mpo,peng2019awr,nair2020awac}, which includes offline-to-online fine-tuning, the improvement step is penalized for deviating from a reference policy $\pref$, typically a policy pretrained on the offline dataset. Maximizing $\E_{x \sim \pi}[Q(\obs,x)] - \lambda\,\kl{\pi(\cdot \mid \obs)}{\pref(\cdot \mid \obs)}$ over all policies yields the \emph{behavior-tilted} optimum
\begin{equation}
    \piplus(x \mid \obs) \;=\; \frac{\pref(x \mid \obs)\exp\!\big(Q(\obs,x)/\lambda\big)}{\mathcal Z(\obs)},
    \qquad
    \mathcal Z(\obs) = \int_{\mathcal{A}} \pref(x \mid \obs)\exp\!\big(Q(\obs,x)/\lambda\big)\,\dt x,
    \label{eq:behavior-boltzman}
\end{equation}
where $\lambda > 0$ trades off the penalty for deviating from $\pref$ against maximizing the expected return.

In \emph{maximum entropy} RL, the reference is dropped and the rewards are instead augmented by the entropy of the current policy $\mathcal{H}(\pi(\cdot\mid \obs)) = -\E_{x \sim \pi}[\log \pi(x \mid \obs)]$, which leads to the soft state-action value function $Q^{\pi}(\obs,x) = \E_{\pi}\big[\sum_{k\ge 0}\gamma^{k} \big(r(\obs_k,x_k) + \lambda\, \mathcal{H}(\pi(\cdot \mid \obs_k))\big) \mid \obs_0 = \obs,\, x_0 = x\big]$. The improvement step of soft policy iteration is the \emph{Boltzmann policy} \citep{ziebart2008maximum,haarnoja2017sql,levine2018rlai,haarnoja2018sac}
\begin{equation}
  \piplus(x \mid \obs) \;=\; \frac{\exp\!\big(Q(\obs,x)/\lambda\big)}{\mathcal Z(\obs)},
  \qquad
  \mathcal Z(\obs) = \int_{\mathcal{A}} \exp\!\big(Q(\obs,x)/\lambda\big)\,\dt x,
  \label{eq:boltzmann}
\end{equation}
with a temperature $\lambda > 0$, which now trades off maximizing the expected return against the policy's stochasticity.

In both cases, the normalization constant $\mathcal Z(\obs)$ is intractable, and generating samples from the potentially complex and multimodal target $\piplus$ requires fitting a tractable policy parameterization to its density. \method{} optimizes a flow-based policy to this end.

\paragraph{Flow policies from flow matching.}
Assume for now that we have access to samples $X_1 \sim \piplus(\cdot \mid \obs)$ from the optimal policy. By drawing independent Gaussian noise $\varepsilon \sim \gauss{\cdot}{0}{I}$, one can construct a stochastic interpolant \citep{albergo2023stochastic} as
\begin{equation}
  X_t \;=\; \alpha(t)\, X_1 \;+\; \beta(t)\, \varepsilon,
  \qquad t \in [0,1],
  \label{eq:interpolant}
\end{equation}
with smooth schedules satisfying $\alpha(0)=0$, $\alpha(1)=1$, $\beta(0)=\sigma_0$, $\beta(1)=0$ (e.g.\ $\alpha(t)=t$, $\beta(t)=\sigma_0(1-t)$). Hence, $X_0 \sim \gauss{\cdot}{0}{\sigma_0^2 I}$ is pure noise, $X_1$ is a clean action from $\piplus(\cdot \mid \obs)$, and $X_t$ in between is a partially noised action with marginal density $\rho_t$. The interpolant defines the \textit{velocity field}
\begin{equation}
  \vel(x,t \mid \obs) \;:=\; \Eint{\piplus}\big[\dot{X}_t \,\big|\, X_t = x\big],
  \qquad \dot X_t = \dot\alpha(t)\, X_1 + \dot\beta(t)\,\varepsilon,
  \label{eq:velocity}
\end{equation}
which represents the mean direction in which interpolating actions move. Here and in the following, the subscript of the expectation lists the random variables of the interpolant it is taken over: $\Eint{\piplus}$ is the expectation over the pair $(X_1, \varepsilon)$ of \cref{eq:interpolant} with clean actions $X_1 \sim \piplus(\cdot \mid \obs)$. Conditioning on $X_t = x$ hence averages over all pairs $(X_1, \varepsilon)$ that are consistent with the noised action $x$ at time $t$. Integrating the ordinary differential equation (ODE)
\begin{equation}
  \dt X_t \;=\; \vel(X_t,t \mid \obs)\,\dt t , \quad X_0 \sim \gauss{\cdot}{0}{\sigma_0^2 I},
  \label{eq:ode}
\end{equation}
transports the prior along the marginals $\rho_t$, so that the terminal samples $X_1$ are distributed according to $\piplus(\cdot \mid \obs)$ \citep{albergo2023stochastic}.
To learn the velocity, one can use \emph{flow matching} \citep{lipman2023flow,liu2023rectified,albergo2023building}, which provides a simple and scalable regression objective, i.e.,
\begin{equation}
  \mathcal{L}_{\mathrm{FM}}(\theta) \;=\; \Eintt{\piplus}\big[\,\big\| \vel_\theta(X_t, t \mid \obs) - \dot X_t \big\|^2\,\big],
  \label{eq:fm}
\end{equation}
where the expectation is additionally taken over $t \sim \mathcal{U}(0,1)$. Optimizing \cref{eq:fm}, however, requires samples from the optimal policy, which are not available in online RL. Instead, one typically has access to a critic $Q$, and in the next section we leverage it to obtain a similar regression objective that can be optimized without samples from $\piplus$.

\section{Fixed-point target flow matching}
\label{sec:tfm}
Now assume we are given a critic function $Q$ with $\piplus(\cdot \mid \obs) \propto \exp(Q(\obs,\cdot)/\lambda)$ instead of samples $X_1$. Furthermore, assume that the critic is differentiable, allowing us to compute the \textit{$Q$-score}
\begin{equation}
  \nabla_x \log \piplus(x \mid \obs) \;=\; \nabla_x Q(\obs,x)/\lambda,
  \label{eq:terminal-score}
\end{equation}
which does not require knowledge of the intractable normalization constant in \cref{eq:boltzmann}. We arrive at the proposed method in two steps. First, we introduce the \textit{target flow identity}, which expresses the velocity as a conditional expectation, similar to flow matching, but in terms of the $Q$-score. The corresponding objective, \textit{target flow matching}, still requires samples $X_1 \sim \piplus$. This is addressed in the second step, where we introduce \textit{fixed-point flow matching (\method{})}, which uses samples of the current policy instead of the optimal one, resulting in a fixed-point iteration with theoretical guarantees.

\paragraph{The target flow identity.} Consider the interpolant in \cref{eq:interpolant} and define
\begin{equation}
  \cs(t) \;:=\; \tfrac{\dot\alpha(t)}{\alpha(t)}\beta^2(t) - \dot\beta(t)\beta(t) \;>\; 0 \quad \text{on } (0,1),
  \label{eq:cs}
\end{equation}
which recurs throughout the paper. The following proposition provides an alternative to flow matching (see \cref{eq:fm}) for representing the velocity field.

\begin{tealbox}
\begin{proposition}[Target flow identity]
\label{prop:tfi}
Let $\piplus(\cdot \mid \obs) \propto \exp(Q(\obs,\cdot)/\lambda)$ with $\mathcal Z(\obs) < \infty$, let $Q(\obs,\cdot)$ be continuously differentiable with $\E_{X_1 \sim \piplus}\big[\|\nabla_x Q(\obs,X_1)\|\big] < \infty$, and let $\vel$ be the velocity field \cref{eq:velocity} of the interpolant \cref{eq:interpolant} with $X_1 \sim \piplus(\cdot \mid \obs)$. Then, for all $t \in (0,1]$,
\begin{equation}
  \vel(x,t \mid \obs)
  \;=\; \tfrac{\dot\alpha(t)}{\alpha(t)}\, x \;+\; \tfrac{\cs(t)}{\lambda\,\alpha(t)}\,
  \Eint{\piplus}\big[\nabla_x Q(\obs, X_1) \,\big|\, X_t = x\big].
  \label{eq:tfi}
\end{equation}
\end{proposition}
\end{tealbox}

Intuitively, the velocity field that generates the Boltzmann policy is, up to a rescaling of the current noisy action, a conditional average of $Q$-scores at clean actions. The proof of~\propref{prop:tfi} is given in \cref{app:proofs} and uses the fact that the velocity can be expressed in terms of the score for Gaussian source distributions in combination with the target score identity \citep{debortoli2024target}.
Note that the factor $1/\alpha(t)$ in \cref{eq:tfi} diverges as $t \to 0$, where $X_t$ is almost pure noise and the $Q$-score at the clean action $X_1$ carries little information about it. The conditional expectation itself remains finite, but a regression target built from \cref{eq:tfi} has large variance near the noise end. In \cref{app:targets} we show that mixing the target with the flow-matching target using a variance-optimal control variate removes this singularity.

Because the $L^2$-optimal predictor is the conditional expectation, \cref{eq:tfi} can directly be turned into a regression objective as stated below.

\begin{tealbox}
\begin{corollary}[Target flow matching]
\label{cor:tfm}
Draw $t \sim \mathcal{U}(0,1)$, $X_1 \sim \piplus(\cdot \mid \obs)$, $\varepsilon \sim \gauss{\cdot}{0}{I}$ and form $X_t$ by \cref{eq:interpolant}, and let $w$ be a time-dependent weighting with $w(t) > 0$ on $(0,1)$. The \emph{target flow matching} objective
\begin{equation}
  \mathcal{L}_{\mathrm{TFM}}(\theta)
  \;=\; \Eintt{\piplus}\Big[w(t)\,\big\|\vel_\theta(X_t, t \mid \obs) - \tfrac{\dot\alpha(t)}{\alpha(t)}\, X_t - \tfrac{\cs(t)}{\lambda\,\alpha(t)}\, \nabla_x Q(\obs, X_1) \big\|^2\Big]
  \label{eq:tfm-objective}
\end{equation}
shares its population minimizer with flow matching, $\mathcal{L}_{\mathrm{FM}}$ of \cref{eq:fm}.
\end{corollary}
\end{tealbox}
In contrast to flow matching, whose target $\dot X_t$ encodes where the data moves, the target of target flow matching encodes where the critic points. However, the main obstacle remains: the expectation in \cref{eq:tfm-objective} still requires samples $X_1 \sim \piplus$.

\paragraph{Fixed-point flow matching.}
While we cannot sample from $\piplus$, we can sample $X_1 \sim \pi_{\theta}(\cdot \mid \obs)$, where $\pi_{\theta}$ denotes the terminal density at $t=1$ obtained by integrating \cref{eq:ode} with the learned velocity field $\vel_{\theta}$ in place of $\vel$. In other words, we can sample from the current policy instead of the optimal one. \method{} leverages this fact by building the interpolant in \cref{eq:tfm-objective} from the policy's own samples. To this end, consider the loss where the interpolant is built from clean actions of an arbitrary flow policy $\pi_{\bar\theta}$ with parameters $\bar\theta$, that is,
\begin{equation}
    \mathcal{L}_{\mathrm{FFM}}(\theta,\bar\theta) = \Eintt{\pi_{\bar\theta}}\Big[w(t)\,\big\|\vel_\theta(X_t, t \mid \obs) - \tfrac{\dot\alpha(t)}{\alpha(t)}\, X_t - \tfrac{\cs(t)}{\lambda\,\alpha(t)}\, \nabla_x Q(\obs, X_1) \big\|^2\Big],
  \label{eq:fp-tfm-loss}
\end{equation}
where the expectation is again over the interpolant \cref{eq:interpolant}, now with $X_1 \sim \pi_{\bar\theta}(\cdot \mid \obs)$ instead of $X_1 \sim \piplus(\cdot \mid \obs)$. Given some initial parameters $\theta_{0}$, \method{} computes the updates
\begin{equation}
  \theta_{k+1}
  \;=\; \argmin_\theta\; \mathcal{L}_{\mathrm{FFM}}(\theta, \theta_{k}),
  \label{eq:fp-tfm}
\end{equation}
which has the optimal policy as a fixed point, as stated below.

\begin{theorem}[Fixed point]
\label{thm:fixedpoint} 
Consider the iteration \cref{eq:fp-tfm} in its population form, that is, with the regression solved exactly over all velocity fields. Under regularity conditions on the critic $Q$ that we state in \cref{app:proofs}, if the current policy is optimal, $\pi_{\theta_k}(\cdot \mid \obs) = \piplus(\cdot \mid \obs)$, then the update returns the velocity field \cref{eq:velocity} of $\piplus$, so that $\pi_{\theta_{k+1}}(\cdot \mid \obs) = \piplus(\cdot \mid \obs)$ and every subsequent update returns the same velocity field. The optimal policy is therefore a fixed point of \method{}.
\end{theorem}

The proof is given in \cref{app:proofs}. In practice, \method{} repeats the following steps. First, clean actions $X_1 \sim \pi_{\theta_k}(\cdot \mid \obs)$ are generated by integrating \cref{eq:ode} with the current velocity field (forward passes only, no gradients). Next, the $Q$-scores $\nabla_x Q(\obs, X_1)$ are computed, which can be stored in a replay buffer for sample reuse. Then, fresh noise $\varepsilon$ and times $t$ are drawn to form the noised actions $X_t$, on which the regression loss \cref{eq:fp-tfm-loss} is minimized; see \cref{fig:method} for an illustration.

\section{Applications in reinforcement learning}
\label{sec:online}

\paragraph{Offline-to-online RL.}
In offline-to-online RL, one commonly pre-trains a flow policy $\velref$ by flow matching \cref{eq:fm} and a critic by an offline method on a fixed dataset. The policy is then fine-tuned online, where it must stay close to the pretrained flow while it improves, since offline critics are reliable only near the data. This motivates the behavior-constrained objective of \cref{sec:prelim}. Prior work for fine-tuning flow policies typically converts the flow to a diffusion model with a specific (memoryless) noise schedule and leverages stochastic optimal control to compute adjoint states, which yield a stable regression objective for training \citep{domingoenrich2025adjoint,li2026qam,dong2026trust}. In contrast, \method{} requires neither a conversion to a diffusion model nor the computation of adjoint states.

Keeping the policy close to the pretrained policy is commonly realized by considering the behavior-tilted target policy $\piplus \propto \pref\,e^{Q/\lambda}$, where $\pref$ is the terminal distribution of the pretrained flow policy \citep{uehara2024finetuning,domingoenrich2025adjoint}. Note that \method{} only requires the score $\scoreref(x) \coloneqq \nabla_x \log \pref(x \mid \obs)$ of the reference policy, which can be read off the pretrained velocity field via the velocity--score bridge (see \cref{app:interpolants}),
\begin{equation}
  \scoreref(x) \;=\; \big(\velref(x,t \mid \obs) - \tfrac{\dot\alpha(t)}{\alpha(t)}\,x\big)\big/\cs(t) \qquad (t \to 1).
  \label{eq:refscore-bridge}
\end{equation}
Similar to \citet{shin2026efficient}, \method{} can then directly be used for behavior-constrained fine-tuning by considering
\begin{equation}
  \mathcal{L}_{\mathrm{BT-FFM}}(\theta, \bar\theta) = \Eintt{\pi_{\bar\theta}}\Big[\big\|\vel_\theta(X_t, t \mid \obs) - \tfrac{\dot\alpha(t)}{\alpha(t)}\, X_t - \tfrac{\cs(t)}{\alpha(t)}\, g(X_1)\big\|^2\Big],
  \label{eq:o2o-target}
\end{equation}
with $g(x) := \scoreref(x) + \tfrac{1}{\lambda}\nabla_x Q(\obs, x)$ the score of the behavior-tilted target, in place of $\mathcal{L}_{\mathrm{FFM}}$ in the fixed-point iteration \cref{eq:fp-tfm}. By the argument of \thmref{thm:fixedpoint} (\cref{app:proofs}), the behavior-tilted policy $\piplus \propto \pref\,e^{Q/\lambda}$ is a fixed point of the resulting iteration.

\paragraph{Maximum entropy reinforcement learning.} The optimal policy in maximum entropy RL, as outlined in \Cref{sec:prelim}, is given by $\piplus(x \mid \obs) \propto{\exp\!\big(Q(\obs,x)/\lambda\big)}$. Hence, we can directly use \method{} to realize the policy improvement step. However, computing the soft $Q$-function requires the entropy of the current flow policy, $\mathcal{H}(\pi_{\theta}(\cdot \mid \obs)) = - \E_{X_1 \sim \pi_{\theta}}[\log \pi_\theta(X_1 \mid \obs)]$. To that end, we use the \emph{instantaneous change of variables} \citep{chen2018neuralode} to compute the log-density of a generated action along its own sampling trajectory as
\begin{equation}
  \log \pi_\theta(X_1 \mid \obs)
  \;=\; \log \gauss{X_0}{0}{\sigma_0^2 I} \;-\; \int_0^1 \nabla_x\!\cdot \vel_\theta(X_t, t \mid \obs)\,\dt t,
  \label{eq:entropy}
\end{equation}
and estimate the entropy by averaging over generated actions. However, the integral in \cref{eq:entropy} requires the divergence of the velocity field in each integration step, which can be computationally expensive. Therefore, we follow common practice \citep{grathwohl2019ffjord} and rely on Hutchinson's trace estimator \citep{hutchinson1990trace,skilling1989eigenvalues}.
More specifically, the estimator computes the divergence as the expectation $\nabla_x\!\cdot \vel = \E_{z}\big[z^{\top}(\partial_x \vel)\,z\big]$ over random probes $z$ with $\E[zz^\top] = I$.
The cost is one vector--Jacobian product per integration step, and no auxiliary density network is needed. For details, including an algorithmic description, we refer to \cref{app:maxent}.

\paragraph{Trust-region constraints.} Recently, on-policy maximum entropy RL has gained popularity, in particular when using massively parallel simulators. For on-policy RL it is well known that trust-region methods stabilize policy optimization and prevent premature convergence \citep{peters2010relative, schulman2015trpo, schulman2017ppo,abdolmaleki2015model, otto2021differentiable, akrour2019projections, reppo2025,trudi2025,amdp2025}. Similar to \citet{amdp2025}, we employ a trust-region loss that keeps the velocity close to the previous one via
\begin{equation}
  \mathcal{L}_{\mathrm{TR}}(\theta, \bar\theta)
  \;=\; \int_0^1 \, \Eint{\pi_{\bar\theta}}\Big[\big\|\vel_{\theta}(X_t,t\mid\obs) - \vel_{\bar\theta}(X_t,t\mid\obs)\big\|^2\Big]\,\dt t.
  \label{eq:klflow}
\end{equation}
\Cref{eq:klflow} is then used to constrain the fixed-point update \cref{eq:fp-tfm} to stay close to the previous flow policy with parameters $\theta_k$, i.e.,
\begin{equation}
\label{eq:tr_ffm}
    \theta_{k+1} = \argmin_{\theta}\  \mathcal{L}_{\mathrm{FFM}}(\theta, \theta_k) \quad \mathrm{s.t.} \quad  \mathcal{L}_{\mathrm{TR}}(\theta, \theta_k) \leq \varepsilon_{\mathrm{TR}}.
\end{equation}
The constrained problem in \cref{eq:tr_ffm} can be approached by the method of Lagrange multipliers \citep{vandenberghe2004convex}. Introducing the multiplier $\nu \ge 0$, the Lagrangian reads
\begin{equation}
  \mathcal{L}_{\mathrm{Lag}}(\theta, \nu; \theta_k)
  \;=\; \mathcal{L}_{\mathrm{FFM}}(\theta, \theta_k) \;+\; \nu\,\big(\mathcal{L}_{\mathrm{TR}}(\theta, \theta_k) - \varepsilon_{\mathrm{TR}}\big),
  \label{eq:lagrangian}
\end{equation}
and \cref{eq:tr_ffm} is replaced by the saddle-point problem
\begin{equation}
  \theta_{k+1} \;=\; \argmin_{\theta}\, \max_{\nu \ge 0}\; \mathcal{L}_{\mathrm{Lag}}(\theta, \nu; \theta_k).
  \label{eq:tr_minmax}
\end{equation}
We solve \cref{eq:tr_minmax} by alternating stochastic gradient descent on $\theta$ with gradient ascent on $\nu$, such that the multiplier grows while the constraint is violated and decays once it is satisfied. We refer to \cref{app:maxent,app:o2o} for the instantiations with the previous policy and with the pretrained flow, respectively, including the dual update.

\paragraph{Correlated exploration.} A flow policy is a deterministic function of its prior draw, and hence, integrating \cref{eq:ode} from $X_0 = \sigma_0\varepsilon$ yields the action $x = \Psi_\theta(\varepsilon \mid \obs)$, whose only source of randomness is $\varepsilon \sim \gauss{\cdot}{0}{I}$. Here, $\Psi_\theta(\cdot \mid \obs)$ is simply the sampler of the flow policy: the map that takes a prior draw $\varepsilon$ and returns the clean action obtained by integrating the ODE with the learned velocity field. Hence, exploration can also be shaped by correlating the prior draws across environment steps. Similar to colored-noise exploration for Gaussian policies \citep{eberhard2023pink}, we draw the priors within an episode from a stationary Ornstein--Uhlenbeck chain with correlation time $\tauou$, i.e.,
\begin{equation}
  \varepsilon_{k+1} \;=\; e^{-1/\tauou}\,\varepsilon_{k} \;+\; \sqrt{1 - e^{-2/\tauou}}\; \xi_{k},
  \qquad \xi_{k} \sim \gauss{\cdot}{0}{I},
  \qquad \varepsilon_{0} \sim \gauss{\cdot}{0}{I},
  \label{eq:ou}
\end{equation}
and execute $x_k = \Psi_\theta(\varepsilon_k \mid \obs_k)$. Since the chain in \cref{eq:ou} is stationary, every $\varepsilon_k$ is marginally distributed according to $\gauss{\cdot}{0}{I}$, and every executed action remains a sample of $\pi_\theta(\cdot \mid \obs_k)$ (see \cref{app:corr-expl}). Therefore, the \method{} regression \cref{eq:fp-tfm-loss}, the entropy estimate \cref{eq:entropy}, and the soft Bellman backup \cref{eq:soft-bellman} are unchanged, and so is the fixed point $\piplus$ of \thmref{thm:fixedpoint}; only the states the agent visits change. Instead of resampling the prior at every step, the agent commits to a direction for roughly $\tauou$ steps, which produces the sustained, directed behavior that sparse-reward tasks require before any reward is observed. We use it only for the hard-exploration tasks of \cref{app:online-results}. Other colored noise processes, such as the pink noise of \citet{eberhard2023pink}, can be used as well, which we leave to future work.

\section{Experiments}
\label{sec:experiments}

We evaluate \method{} in two settings: offline-to-online fine-tuning, and online on-policy maximum entropy RL, where we train a flow policy from scratch. 
For the offline-to-online setting, we focus on the ten OGBench \citep{park2025ogbench} locomotion and manipulation domains totaling 50 tasks. 
For the online RL setting, we consider the GPU-accelerated MJX \citep{mujoco_playground_2025} Playground and humanoid tasks, which result in 24 tasks in total.
We restrict ourselves to presenting the main results in this section, and refer the reader to \cref{app:o2o} and to \cref{app:maxent} for additional results and ablations, such as varying the number of integration steps, the interpolation schedules, or different trust region constraints.

\paragraph{Offline-to-online RL.}
We pretrain a behavior-cloning flow policy and a critic offline (to $1$M gradient steps) and then fine-tune online (to $1.5$M), against TRQAM \citep{dong2026trust}, QAM (+QAM-E) \citep{li2026qam}, FQL (+IFQL) \citep{park2025fql}, DSRL \citep{wagenmaker2025dsrl}, and CGQL-L \citep{dong2026trust} under an identical protocol (\cref{fig:curves-agg}). \method{}'s offline phase is deliberately conservative, keeping the policy inside the pretrained flow's support instead of chasing an offline critic that is unreliable off-support, so its offline average ($0.63$) trails the strongest offline baseline (TRQAM, $0.68$; \cref{tab:offline}). However, once the online phase starts and the interaction makes the critic trustworthy, the budget is relaxed and \method{} improves rapidly and attains the best aggregate success at $1.5$M steps ($0.95$ mean vs.\ $0.93$ for the next-best method; \cref{tab:online}), with the widest margins on the hard, long-horizon tasks (e.g.\ HumanoidMaze-Large, $0.99$ vs.\ $0.76$) and without a single divergence across seeds.

\begin{figure}[t]\centering
  \resizebox{\textwidth}{!}{\input{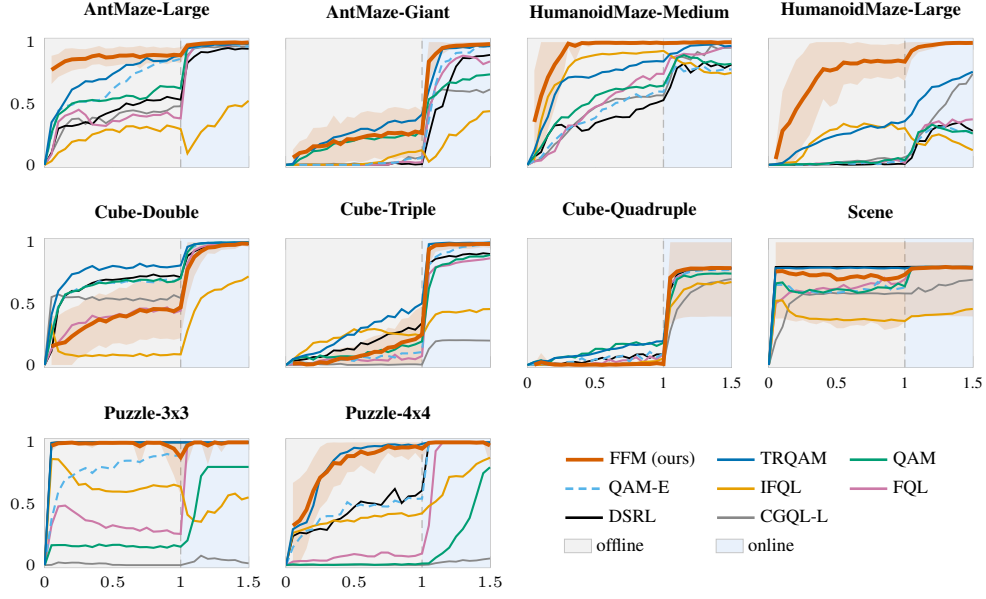}}
  \caption{\textbf{Per-domain aggregated offline-to-online success rate} (mean over the five tasks
  and all seeds). Offline ($\le\!1$M) and online ($>\!1$M) phases are shaded; \method{}
  (ours) is drawn in bold with a $\pm1$ standard-deviation band. Per-task curves and full tables are in \cref{app:exp}.}
  \label{fig:curves-agg}
\end{figure}

\paragraph{Online RL.}
We additionally train \method{} in the online on-policy RL setting where the policy is trained from scratch.
Concretely, we train \method{} on the 20 MuJoCo Playground DMC \citep{mujoco_playground_2025} tasks and on the four MuJoCo Playground humanoid tasks that require precise high-dimensional control.
We compare \method{}'s performance against state-of-the-art Gaussian policy representations (REPPO \citep{reppo2025}, PPO \citep{schulman2017ppo}, SPO \citep{xie2025simple}), diffusion policy representations (TRUDI \citep{trudi2025}, AMDP \citep{amdp2025}, DIME \citep{celik2025dime}, DPPO \citep{ren2025dppo}) and flow policy representations (FPO \citep{mcallister2026flow}).
Following \citet{agarwal2021deep}, every learning curve represents the aggregated interquartile mean (IQM) with $95\%$ confidence intervals from a stratified bootstrap for every task in \cref{fig:dmc-grid} for the Playground tasks and in \cref{fig:humanoid-grid} for the humanoid tasks, where every task is run on $10$ seeds. 
While \method{} performs comparably on the lower-dimensional Playground tasks (\cref{fig:dmc-aggregate} left two panels and \cref{fig:dmc-grid}), \method{} converges faster to the highest performance alongside TRUDI on the more sophisticated, high-dimensional humanoid tasks.
This indicates that \method{}'s benefits are most visible in high-dimensional tasks.
We further integrated correlated exploration into \method{} and ran experiments on the Playground tasks. 
The resulting learning curves are shown in \cref{fig:dmc-aggregate} left two panels and \cref{fig:dmc-grid} for the aggregated and per-task results, respectively. 
Please note that we have run these experiments only on a subset of the tasks (see \cref{fig:dmc-grid}). 
The aggregated \method{} + OU curve (dashed in \cref{fig:dmc-aggregate}) uses the \method{} + OU learning curve on each task where we ran correlated exploration and the \method{} learning curve on all other tasks.
\method{} + OU shows a faster convergence to a similar return value as the highest-performing baseline DIME, indicating that if computational resources permit, \method{}'s performance can be boosted even more by tuning $\tauou$.

\begin{figure}[t]\centering
  \includegraphics[width=\textwidth]{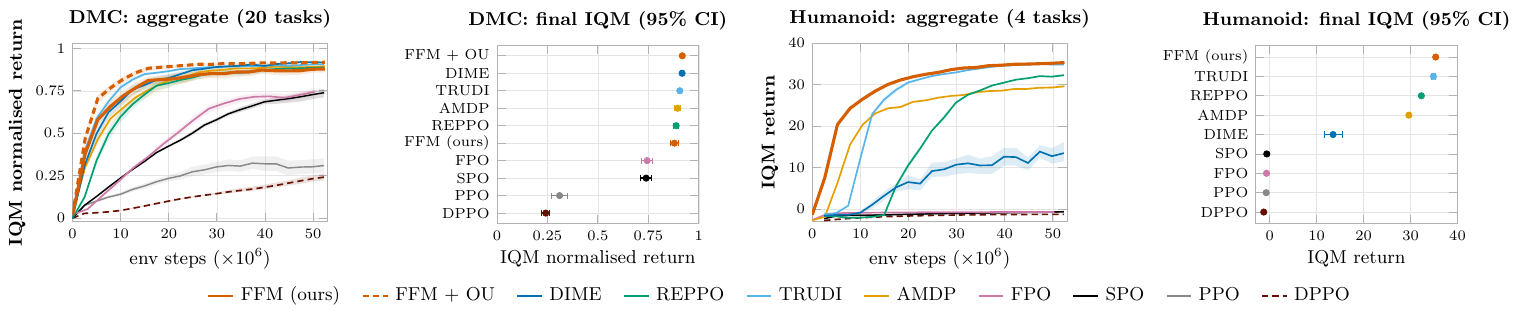}
  \caption{\textbf{Maximum Entropy Online RL.} We have run \method{} on all 20 MJX Playground tasks and on the four MJX humanoid tasks and present the aggregated results over all Playground tasks (left two panels) and the aggregated results over the humanoid tasks (right two panels). 
  We have evaluated all methods on $10$ seeds for every single environment (see \cref{fig:dmc-grid} and \cref{fig:humanoid-grid} for results) and report the IQM with $95\%$ stratified bootstrapped confidence intervals following \citet{agarwal2021deep}. 
  \method{} performs on par with the baselines on the Playground tasks (left two panels) but converges to the highest return much faster than TRUDI on the high-dimensional humanoid tasks. 
  This result indicates that \method{}'s benefits are most pronounced on high-dimensional tasks.
  Furthermore, we can boost \method{}'s performance by leveraging its deterministic mapping properties and incorporating correlated exploration (\method{} + OU), which leads to a significantly faster convergence to a similar return value as the best-performing baseline DIME (left two panels) (see \cref{fig:ou-grid}).} 
    \label{fig:dmc-aggregate}
\end{figure}

\section{Conclusion}
\label{sec:conclusion}

We introduced fixed-point flow matching, a data-free procedure for training flow policies using a flow-matching-style regression objective. Its basis is the target flow identity, which expresses the velocity field as a conditional average of $Q$-scores at clean actions and thereby replaces the data-derived targets of flow matching by critic gradients. Evaluating the resulting objective on samples from the current policy instead of the optimal policy yields a fixed-point iteration that has the optimal policy as a fixed point. The construction is conceptually simple, easy to implement, and needs no gradients through the sampling chain, no conversion of the flow into a stochastic differential equation, and no adjoint simulation. We consider two instantiations of the proposed framework: an online maximum entropy actor--critic algorithm, and offline-to-online fine-tuning against a behavior-tilted target. Across $50$ offline-to-online tasks and $24$ online tasks, both are competitive with or better than the strongest available baselines.

\paragraph{Limitations.} \method{} requires a \emph{differentiable} target. The regression target is built from the $Q$-score $\nabla_x Q(\obs, x)$, so the critic must be differentiable in the action and the policy update is only as accurate as the critic's gradient. This restricts \method{} to continuous action spaces. Extensions to non-differentiable targets and discrete action spaces are left for future work.

\section*{Acknowledgements}
The authors thank Yonghoon Dong for providing the results for several baselines used in \cref{fig:curves-agg}. D.B. acknowledges support by funding from a Google PhD fellowship in Machine Learning and ML Foundations, the pilot program Core Informatics of the Helmholtz Association (HGF) and the state of Baden-Württemberg through bwHPC, as well as the HoreKa supercomputer funded by the Ministry of Science, Research and the Arts Baden-Württemberg and by the German Federal Ministry of Education and Research. Furthermore, all authors are supported by the European Research Council (ERC) under the European Union’s Horizon Europe program through the project SMARTI³ (Grant Agreement No. 101171393).

\newpage

\bibliography{references}
\bibliographystyle{iclr2026_conference}

\clearpage
\appendix

\startcontents[appendices]
{\noindent\textbf{\large Appendix contents}\par\nobreak\vspace{4pt}
 \small\printcontents[appendices]{}{1}{\setcounter{tocdepth}{2}}}
\vspace{6pt}
\hrule
\vspace{10pt}

\newpage
\section{Proofs}
\label{app:proofs}
This appendix proves the three results of \cref{sec:tfm}: the target flow identity (\propref{prop:tfi}), the target flow matching objective (\corref{cor:tfm}), and the fixed-point property of \method{} (\thmref{thm:fixedpoint}). We use the notation of \cref{app:interpolants} and suppress the state $\obs$.

\paragraph{Proof of the target flow identity (\texorpdfstring{\propref{prop:tfi}}{Proposition})}
The proof combines the velocity--score bridge \cref{eq:app-bridge} with the \emph{target score identity} of \citet{debortoli2024target}: for a continuously differentiable density $\mu$ with $\E_\mu\big[\|\nabla\log\mu(X_1)\|\big] < \infty$ and all $t \in (0,1]$,
\begin{equation}
  \score^{\mu}(x,t)
  \;=\; \frac{1}{\alpha(t)}\, \E_{\mu}\big[\nabla \log \mu(X_1) \,\big|\, X_t = x\big],
  \label{eq:tsi}
\end{equation}
where $\score^\mu$ denotes the score \cref{eq:score-def} of the interpolant of $\mu$. Whereas the noise identity \cref{eq:app-tweedie} averages the noise, \cref{eq:tsi} averages the score of the \emph{target} at clean samples; it follows from an integration by parts that moves the gradient from the Gaussian kernel of $\rho_t$ onto $\mu$, and the factor $1/\alpha(t)$ accounts for the scaling of $X_1$ in \cref{eq:app-interpolant}.

\begin{proof}[Proof of \propref{prop:tfi}]
Apply \cref{eq:tsi} to $\mu = \piplus \propto e^{Q(\obs,\cdot)/\lambda}$. By \cref{eq:terminal-score}, $\nabla_x \log\piplus = \nabla_x Q(\obs,\cdot)/\lambda$, and the integrability assumption of \propref{prop:tfi} is exactly that of the identity, so the score of the noised Boltzmann policy is
\begin{equation}
  \score^{\piplus}(x,t)
  \;=\; \frac{1}{\lambda\,\alpha(t)}\, \E_{\piplus}\big[\nabla_x Q(\obs, X_1) \,\big|\, X_t = x\big].
  \label{eq:qscore}
\end{equation}
Inserting \cref{eq:qscore} into the bridge \cref{eq:app-bridge} gives \cref{eq:tfi}.
\end{proof}

\paragraph{Proof of target flow matching (\texorpdfstring{\corref{cor:tfm}}{Corollary})}
\begin{proof}[Proof of \corref{cor:tfm}]
Both objectives are weighted least-squares regressions of $\vel_\theta(X_t,t)$ onto a random target under the interpolant of $\piplus$ and $t \sim \mathcal U(0,1)$: flow matching \cref{eq:fm} onto $\dot X_t$, and target flow matching \cref{eq:tfm-objective} onto
\begin{equation}
  Y \;:=\; \tfrac{\dot\alpha(t)}{\alpha(t)}\, X_t + \tfrac{\cs(t)}{\lambda\,\alpha(t)}\, \nabla_x Q(\obs, X_1).
\end{equation}
When the minimization ranges over all velocity fields, the minimizer of a least-squares regression is the conditional expectation of the target given the regressor, here $(X_t, t)$, for almost every $(x,t)$ with respect to the interpolant marginals; since the weighting $w(t) > 0$ depends on $t$ alone, it does not change this pointwise minimizer. For flow matching the conditional expectation is $\E_{\piplus}[\dot X_t \mid X_t = x] = \vel(x,t)$ by the definition \cref{eq:velocity} of the velocity. For target flow matching, the first term of $Y$ is a function of $X_t$ and therefore
\begin{equation}
  \E_{\piplus}[Y \mid X_t = x]
  \;=\; \tfrac{\dot\alpha(t)}{\alpha(t)}\, x + \tfrac{\cs(t)}{\lambda\,\alpha(t)}\, \E_{\piplus}\big[\nabla_x Q(\obs, X_1) \,\big|\, X_t = x\big]
  \;=\; \vel(x,t)
\end{equation}
by the target flow identity \cref{eq:tfi}. Both objectives hence share the population minimizer $\vel$, for every positive weighting $w$.
\end{proof}

\paragraph{Proof of the fixed point (\texorpdfstring{\thmref{thm:fixedpoint}}{Theorem})}

We first state \thmref{thm:fixedpoint} in full. The regularity conditions the main text defers concern only the critic.

\begin{theorem*}[\thmref{thm:fixedpoint}, full statement]
Let $Q(\obs,\cdot)$ be continuously differentiable, let $\piplus \propto e^{Q/\lambda}$ with $\mathcal Z(\obs) = \int e^{Q/\lambda}\,\dt x < \infty$ have finite second moments and $\E_{\piplus}\big[\|\nabla_x Q(\obs, X_1)\|^2\big] < \infty$, and let $w(t) > 0$ on $(0,1)$. Consider the iteration \cref{eq:fp-tfm} in its population form, in which the regression is solved exactly over all velocity fields at every $t \in (0,1)$. If $\pi_{\theta_k}(\cdot \mid \obs) = \piplus(\cdot \mid \obs)$, then $\vel_{\theta_{k+1}}$ is the velocity field $\vel$ of \cref{eq:velocity}, $\pi_{\theta_{k+1}}(\cdot \mid \obs) = \piplus(\cdot \mid \obs)$, and $\vel_{\theta_{k+m}} = \vel_{\theta_{k+1}}$ for all $m \geq 1$.
\end{theorem*}

The two moment conditions make the regression target of \cref{eq:fp-tfm-loss} square integrable at every $t \in (0,1)$, so that the regression at each time is well posed, and they imply the integrability assumption of \propref{prop:tfi}.

\begin{proof}[Proof of \thmref{thm:fixedpoint}]
The loss \cref{eq:fp-tfm-loss} depends on the current parameters $\theta_k$ only through the distribution $\pi_{\theta_k}(\cdot \mid \obs)$ of the clean actions $X_1$. For $\pi_{\theta_k} = \piplus$ it therefore coincides with the target flow matching objective \cref{eq:tfm-objective}, whose minimizer over all velocity fields is the velocity field $\vel$ of the interpolant of $\piplus$ by \corref{cor:tfm}; hence $\vel_{\theta_{k+1}} = \vel$. Integrating \cref{eq:ode} with $\vel$ transports the prior to $\piplus$ (\cref{app:interpolants}), so $\pi_{\theta_{k+1}} = \piplus$. The next update therefore again minimizes \cref{eq:tfm-objective} and returns $\vel$, and by induction does every later one.
\end{proof}

\paragraph{The behavior-tilted target}
The argument applies verbatim to the offline-to-online objective \cref{eq:o2o-target}. There the target is $\piplus \propto \pref\, e^{Q/\lambda}$, whose terminal score is $g = \scoreref + \nabla_x Q/\lambda$, and the regression target of \cref{eq:o2o-target} is $\tfrac{\dot\alpha(t)}{\alpha(t)} X_t + \tfrac{\cs(t)}{\alpha(t)}\, g(X_1)$. Applying the target score identity \cref{eq:tsi} to $\mu = \piplus$, with $\nabla\log\piplus = g$ since the normalizer of $\piplus$ is constant in $x$, and inserting it into the bridge \cref{eq:app-bridge} shows that under sampling from $\piplus$ the conditional expectation of this target is the velocity field of $\piplus$. The proof above then goes through with $\nabla_x Q/\lambda$ replaced by $g$, under the conditions of the theorem with $Q$ replaced by $Q + \lambda\log\pref$.

\section{Further method details}
\label{app:details}

This appendix collects details of the method that the main text defers: the background on one-sided stochastic interpolants that the main text and the proofs of \cref{app:proofs} rely on (\cref{app:interpolants}), the family of regression targets that mixes the \method{} target with the flow-matching target and its variance-optimal member (\cref{app:targets}), and the correlated exploration scheme of \cref{sec:online} (\cref{app:corr-expl}).

\subsection{One-sided stochastic interpolants}
\label{app:interpolants}

This subsection collects the background on one-sided stochastic interpolants \citep{albergo2023stochastic} that the main text and the proofs of \cref{app:proofs} rely on: the score of the noised marginals, the denoising identities that express it as a conditional expectation, the velocity--score bridge that eliminates the score from all main-text formulas, and the deterministic and stochastic samplers that share the marginals of the flow. All identities are standard and we state them without proof. Throughout, $\mu$ is a probability density on $\R^d$ (an action distribution; the conditioning state $\obs$ is suppressed), $X_1 \sim \mu$ and $\varepsilon \sim \gauss{\cdot}{0}{I}$ are independent, and
\begin{equation}
  X_t = \alpha(t) X_1 + \beta(t)\varepsilon,
  \qquad
  \alpha(0)=0,\;\; \alpha(1)=1,\;\; \beta(0)=\sigma_0,\;\; \beta(1)=0,
  \label{eq:app-interpolant}
\end{equation}
with $\alpha, \beta$ continuously differentiable, $\alpha(t) > 0$ on $(0,1]$ and $\beta(t) > 0$ on $[0,1)$. Expectations $\E_\mu[\cdot]$ are over the pair $(X_1, \varepsilon)$, that is, over the interpolant of $\mu$; they abbreviate the main-text notation $\Eint{\mu}$ of \cref{eq:velocity}. We write $\rho_t$ for the density of $X_t$, which is the Gaussian convolution $\rho_t(x) = \int \mu(x_1)\, \varphi_{\beta^2(t)}(x - \alpha(t) x_1)\,\dt x_1$ with $\varphi_{\beta^2(t)}$ the density of $\gauss{\cdot}{0}{\beta^2(t) I}$, and we call
\begin{equation}
  \score(x,t) \;:=\; \nabla_x \log\rho_t(x)
  \label{eq:score-def}
\end{equation}
the \emph{score} of the noised marginals. It is the second field, next to the velocity \cref{eq:velocity}, around which the constructions of diffusion and flow models revolve.

\paragraph{Denoising identities}
Since $X_t \mid X_1 \sim \gauss{\cdot}{\alpha(t) X_1}{\beta^2(t) I}$, the score of the Gaussian-smoothed marginal is a conditional average of the noise, and the posterior mean of the clean sample is an affine function of the score. Written for the interpolant \cref{eq:app-interpolant}, Tweedie's formula \citep{robbins1956empirical,efron2011tweedie} reads
\begin{equation}
  \score(x,t) = -\frac{1}{\beta(t)}\,\E_\mu[\varepsilon \mid X_t = x],
  \qquad
  \hat x_1(x,t) := \E_\mu[X_1 \mid X_t = x] = \frac{x + \beta^2(t)\, \score(x,t)}{\alpha(t)},
  \label{eq:app-tweedie}
\end{equation}
where the second identity follows from the first by taking the conditional expectation of $x = \alpha(t) X_1 + \beta(t)\varepsilon$ given $X_t = x$. We refer to the first identity as the \emph{noise identity} and to $\hat x_1$ as the \emph{denoiser}; the noise identity is also what turns score matching into the denoising regression of \citet{vincent2011dsm}.

\paragraph{The velocity--score bridge and sampling}
The velocity \cref{eq:velocity} is the conditional mean of $\dot X_t = \dot\alpha(t) X_1 + \dot\beta(t) \varepsilon$, so inserting the two identities of \cref{eq:app-tweedie} expresses it through the score,
\begin{equation}
  \vel(x,t)
  = \dot\alpha(t)\, \E_\mu[X_1 \mid X_t = x] + \dot\beta(t)\, \E_\mu[\varepsilon \mid X_t = x]
  = \frac{\dot\alpha(t)}{\alpha(t)}\, x + \cs(t)\, \score(x,t),
  \label{eq:app-bridge}
\end{equation}
with the coefficient $\cs(t) = \tfrac{\dot\alpha(t)}{\alpha(t)}\beta^2(t) - \dot\beta(t)\beta(t) > 0$ of \cref{eq:cs}, and conversely
\begin{equation}
  \score(x,t) \;=\; \frac{1}{\cs(t)}\Big(\vel(x,t) - \tfrac{\dot\alpha(t)}{\alpha(t)}\, x\Big).
  \label{eq:app-bridge-inv}
\end{equation}
For the one-sided interpolant either field therefore determines the other pointwise \citep{albergo2023stochastic}: the \emph{velocity--score bridge}. \Cref{eq:app-bridge-inv} is how the main text reads the reference's terminal score off a pretrained velocity field in \cref{eq:refscore-bridge}; since $\cs(t) \to 0$ as $t \to 1$, the score is evaluated at a time close to, but not at, the clean endpoint.

The marginals $\rho_t$ solve the continuity equation $\partial_t \rho_t + \nabla\!\cdot\!(\rho_t\, \vel) = 0$ \citep{albergo2023stochastic}, so integrating the ODE \cref{eq:ode} from the prior $\rho_0 = \gauss{\cdot}{0}{\sigma_0^2 I}$ transports it to $\rho_1 = \mu$: this is the flow policy of the main text. The same marginals are shared by the family of SDEs
\begin{equation}
  \dt X_t \;=\; \Big(\vel + \tfrac{\sigma^2(t)}{2}\, \score\Big)(X_t,t)\,\dt t + \sigma(t)\, \dt W_t,
  \qquad X_0 \sim \gauss{\cdot}{0}{\sigma_0^2 I},
  \label{eq:sde}
\end{equation}
for any diffusion coefficient $\sigma(t) \geq 0$, because the additional drift $\tfrac{\sigma^2(t)}{2}\score$ cancels the diffusion term in the Fokker--Planck equation and leaves the continuity equation of the flow \citep{song2021sde,albergo2023stochastic}. A trained velocity field can hence be sampled either deterministically by \cref{eq:ode} or stochastically by \cref{eq:sde}, with the score supplied by the bridge \cref{eq:app-bridge-inv}, and the time reversal of \cref{eq:sde}, with drift $\vel - \tfrac{\sigma^2(t)}{2}\score$ and run from $t = 1$ to $0$, transports $\mu$ back to the prior \citep{song2021sde}. We use the ODE throughout.

\subsection{Generalized regression targets and control variates}
\label{app:targets}

The main text regresses the velocity head onto the target of \cref{eq:tfm-objective}, in which the critic gradient is, besides $X_t$, the only random quantity. This subsection places that target in a one-parameter family that mixes it with the flow-matching target, and states the member of the family with the smallest variance. The family is most transparent in score space, where it is the generalized target score identity of the bridge matching sampler \citep{blessing2026bms}, and transfers to velocity space through the bridge \cref{eq:app-bridge}; the velocity form is what our offline-to-online implementation uses (\cref{app:exp}). We state the results without proofs, which are given in \citet{blessing2026bms} and follow the pattern of \cref{app:proofs}.

\paragraph{The generalized target score identity}

The noise identity \cref{eq:app-tweedie} and the target score identity \cref{eq:tsi} express the same score $\score^{\mu}(x,t)$ as conditional expectations of two different random variables, the scaled noise $-\varepsilon/\beta(t)$ and the scaled target score $\nabla\log\mu(X_1)/\alpha(t)$. Any convex combination of the two is therefore again conditionally unbiased: for a mixing schedule $c(t) \in [0,1]$,
\begin{equation}
  \score^{\mu}(x,t)
  \;=\; \E_\mu\big[\zeta(X_1,\varepsilon,t) \,\big|\, X_t = x\big],
  \qquad
  \zeta(X_1,\varepsilon,t) \;:=\; \frac{c(t)}{\alpha(t)}\, \nabla \log \mu(X_1) \;-\; \frac{1 - c(t)}{\beta(t)}\, \varepsilon,
  \label{eq:gtsi}
\end{equation}
which is the \emph{generalized target score identity} of \citet{blessing2026bms}; $c \equiv 1$ recovers \cref{eq:tsi} and $c \equiv 0$ the noise identity. The two random variables have complementary variances: the target-score term carries the factor $1/\alpha(t)$, which diverges at the noise end $t \to 0$, and the noise term the factor $1/\beta(t)$, which diverges at the clean end $t \to 1$. Since the conditional mean of the noise term is known, it acts as a \emph{control variate} for the target-score term, and the mixing schedule can be chosen to minimize the variance of the combined estimator. For a schedule that depends on $t$ alone, the relevant quantity is the variance averaged over $X_t$, $\E_\mu\big[\|\zeta - \score^{\mu}(X_t,t)\|^2\big] = \E_\mu\big[\|\zeta\|^2\big] - \E_\mu\big[\|\score^{\mu}(X_t,t)\|^2\big]$, whose second term does not depend on $c$. Because $X_1$ and $\varepsilon$ are independent, the cross term of the first vanishes and
\begin{equation}
  \E_\mu\big[\|\zeta(X_1,\varepsilon,t)\|^2\big]
  \;=\; d\,\Big(\frac{c^2(t)}{\alpha^2(t)}\,\bar{\mathcal I} + \frac{(1 - c(t))^2}{\beta^2(t)}\Big),
  \qquad
  \bar{\mathcal I} := \tfrac{1}{d}\,\E_{\mu}\big[\|\nabla\log\mu(X_1)\|^2\big],
  \label{eq:zeta-var}
\end{equation}
where $\bar{\mathcal I}$ is the per-coordinate Fisher information of $\mu$. Minimizing \cref{eq:zeta-var} over $c(t)$ yields \citep{blessing2026bms}
\begin{equation}
  c^\star(t) \;=\; \frac{\alpha^2(t)}{\alpha^2(t) + \beta^2(t)\, \bar{\mathcal I}},
  \label{eq:cstar}
\end{equation}
for any target $\mu$ with finite Fisher information; in practice, $\bar{\mathcal I}$ is estimated on the current batch. Its shape is that of an inverse signal-to-noise ratio: the target-score term dominates near $t = 1$, where samples are clean, and the noise anchor near $t = 0$.

\paragraph{The generalized target flow identity}
The velocity image of the family \cref{eq:gtsi} has a direct reading. Two velocity targets are available for the interpolant \cref{eq:app-interpolant}, the flow-matching target $\dot X_t$ of \cref{eq:fm} and the target-flow-identity target of \cref{eq:tfm-objective}, and their convex combination is
\begin{equation}
  Y_c(X_1, \varepsilon, t)
  \;:=\; \big(1 - c(t)\big)\, \dot X_t
  \;+\; c(t)\, \Big(\tfrac{\dot\alpha(t)}{\alpha(t)}\, X_t + \tfrac{\cs(t)}{\lambda\,\alpha(t)}\, \nabla_x Q(\obs, X_1)\Big),
  \qquad c(t) \in [0,1].
  \label{eq:velmix}
\end{equation}
Under sampling from the Boltzmann policy, both endpoints have the same conditional mean, the velocity field generating $\piplus$ (\cref{eq:velocity} for $c \equiv 0$ and \propref{prop:tfi} for $c \equiv 1$), which gives the \emph{generalized target flow identity}
\begin{equation}
  \vel(x,t \mid \obs) \;=\; \Eint{\piplus}\big[\, Y_c(X_1, \varepsilon, t) \,\big|\, X_t = x\big]
  \qquad \text{for every } c(t) \in [0,1].
  \label{eq:gtfi}
\end{equation}
Substituting $\dot X_t = \dot\alpha(t) X_1 + \dot\beta(t)\varepsilon$, $X_t = \alpha(t) X_1 + \beta(t)\varepsilon$, and the schedule identity $\tfrac{\dot\alpha(t)}{\alpha(t)}\beta(t) = \dot\beta(t) + \tfrac{\cs(t)}{\beta(t)}$ into \cref{eq:velmix} shows that the family is exactly the velocity image of the score-space target $\zeta$ of \cref{eq:gtsi},
\begin{equation}
  Y_c
  \;=\; \tfrac{\dot\alpha(t)}{\alpha(t)}\, X_t \;+\; \cs(t)\,\zeta(X_1,\varepsilon,t),
  \label{eq:velmix-equiv}
\end{equation}
so \cref{eq:gtfi} is the bridge applied to \cref{eq:gtsi}. By \cref{eq:gtfi}, the proof of \thmref{thm:fixedpoint} carries over to the mixed target: under sampling from $\piplus$ the regression onto $Y_c$ returns the velocity field of $\piplus$, so $\piplus$ remains a fixed point for every mixing schedule $c(t) \in [0,1]$. The map $\zeta \mapsto Y_c$ is affine with a shift that depends only on $X_t$ and a scale $\cs(t)$ that does not depend on $c$, so $\mathrm{Var}(Y_c \mid X_t = x) = \cs^2(t)\, \mathrm{Var}(\zeta \mid X_t = x)$ and the two families share their variance-minimizing schedule:
\begin{equation}
  c^\star(t) \;=\; \frac{\alpha^2(t)}{\alpha^2(t) + \beta^2(t)\, \bar{\mathcal I}},
  \qquad
  \bar{\mathcal I} = \tfrac{1}{d}\,\E_q\big[\|\nabla_x Q/\lambda\|^2\big],
  \label{eq:velmix-cstar}
\end{equation}
is the variance-optimal control variate also in velocity space. It also removes the endpoint singularity noted in \cref{sec:tfm}: in \cref{eq:velmix} the critic gradient enters with the coefficient $c(t)\,\cs(t)/(\lambda\,\alpha(t))$, and since $c^\star(t) = \Theta(\alpha^2(t))$ as $t \to 0$ while $\alpha(t)\,\cs(t) \to \dot\alpha(0)\,\sigma_0^2 < \infty$, this coefficient stays bounded at the noise end, whereas $c^\star(t) \to 1$ as $t \to 1$ hands the target back to the pure \method{} form where actions are clean. No time truncation $t \ge \epsilon_0$ is needed; our offline-to-online implementation uses this schedule (\cref{app:exp}).

\subsection{Correlated exploration}
\label{app:corr-expl}

This subsection expands on the correlated exploration scheme of \cref{sec:online}.

\paragraph{The flow policy as a deterministic map}
Integrating the ODE \cref{eq:ode} with the learned velocity field from the prior draw $X_0 = \sigma_0\varepsilon$ up to $t = 1$ is a deterministic computation once $\varepsilon$ and the state $\obs$ are fixed. We write $\Psi_\theta(\varepsilon \mid \obs)$ for the resulting clean action, so that $\Psi_\theta(\cdot \mid \obs)$ is the sampler of the flow policy: it maps a prior draw to an action, and $\pi_\theta(\cdot \mid \obs)$ is the law of $\Psi_\theta(\varepsilon \mid \obs)$ for $\varepsilon \sim \gauss{\cdot}{0}{I}$. In particular, an action is a sample of $\pi_\theta(\cdot \mid \obs)$ whenever the prior draw it is computed from is distributed according to $\gauss{\cdot}{0}{I}$, regardless of how that draw was generated.

\paragraph{The Ornstein--Uhlenbeck chain is stationary}
The chain \cref{eq:ou} is a first-order autoregressive process with coefficient $a := e^{-1/\tauou} \in (0,1)$, that is, $\varepsilon_{k+1} = a\,\varepsilon_k + \sqrt{1 - a^2}\,\xi_k$. If $\varepsilon_k \sim \gauss{\cdot}{0}{I}$ and $\xi_k \sim \gauss{\cdot}{0}{I}$ are independent, then $\varepsilon_{k+1}$ is Gaussian with mean zero and covariance
\begin{equation}
  \E\big[\varepsilon_{k+1}\varepsilon_{k+1}^\top\big] \;=\; a^2\, I + (1 - a^2)\, I \;=\; I,
  \label{eq:ou-stationary}
\end{equation}
so $\varepsilon_{k+1} \sim \gauss{\cdot}{0}{I}$ as well. Since $\varepsilon_0 \sim \gauss{\cdot}{0}{I}$, induction over $k$ shows that every $\varepsilon_k$ is marginally standard normal, and by the previous paragraph every executed action $x_k = \Psi_\theta(\varepsilon_k \mid \obs_k)$ is a sample of $\pi_\theta(\cdot \mid \obs_k)$. What changes is the dependence between draws: unrolling the recursion gives
\begin{equation}
  \E\big[\varepsilon_{k+j}\varepsilon_{k}^\top\big] \;=\; a^{j}\, I \;=\; e^{-j/\tauou}\, I, \qquad j \ge 0,
  \label{eq:ou-autocorr}
\end{equation}
so the correlation between two prior draws decays exponentially with their distance in time, with $\tauou$ steps as the time constant. The two limits recover familiar schemes: $\tauou \to 0$ gives $a \to 0$ and independent draws at every step, which is the usual way of sampling the policy, and $\tauou \to \infty$ gives $a \to 1$ and a single draw $\varepsilon_0$ that is kept for the whole episode, similar to the state-dependent exploration of \citet{raffin2022gsde}.

\section{Extended related work}
\label{app:extended-rw}

This appendix complements \cref{sec:rw}. The first paragraph places \method{} in the family of matching objectives that train diffusion and flow models from an unnormalized target rather than from samples, which is where its fixed-point structure comes from. The second collects work on critic-driven updates of diffusion and flow policies that \cref{sec:rw} groups but has no room to discuss individually.

\paragraph{Fixed-point methods for diffusion and flow models.}
Training a diffusion or flow model from an unnormalized density instead of samples is the problem of neural samplers. Early approaches minimize a path-space divergence by simulating the generative process \citep{zhang2022pis,vargas2023dds}. Matching-based samplers instead regress onto a target computed from the energy and its gradient: iDEM \citep{akhoundsadegh2024idem}, iEFM \citep{woo2024iefm}, BNEM \citep{ouyang2026bnem}, and EWFM \citep{dern2025ewfm} estimate the noised score, velocity, or energy by importance-weighted Monte Carlo, target score matching \citep{debortoli2024target}, on which \method{} builds, writes the noised score as a conditional expectation of the clean-data score but requires samples of the target, and tilt matching \citep{potaptchik2025tilt} matches the velocity of a reward-tilted interpolant without reward gradients. These methods are iterative: the target is evaluated with samples of the current model and the regression is repeated, which is the fixed-point structure that \thmref{thm:fixedpoint} makes explicit for critic-driven targets. Adjoint matching \citep{domingoenrich2025adjoint} forms a second family of such regressions. Proposed for reward fine-tuning, it has since been used for sampling from unnormalized densities \citep{havens2025adjointsampling,liu2025asbs,mosso2026riemannian}, for discrete generative models \citep{so2026dam}, for faster fine-tuning pipelines \citep{guo2026improved,shin2026efficient}, and in RL for policy extraction \citep{li2026qam,dong2026trust,ghanem2026eram} and maximum entropy RL \citep{amdp2025}; its theory has been connected to the stochastic maximum principle \citep{domingoenrich2026smp} and unified with sampling \citep{domingoenrich2026unified}. The bridge matching sampler \citep{blessing2026bms} shows that these matching objectives, including adjoint sampling and ASBS, are fixed-point iterations rooted in Nelson's relation and stabilizes them by damping. \method{} belongs to this family with two differences: the target flow identity provides the regression target from critic gradients at clean actions, so no adjoint ODE is needed, and the target itself moves as the critic is learned.

\paragraph{Further critic-driven policy updates.}
Sorted by how the critic enters the actor update, the methods of \cref{sec:rw} fall into a few families, and we record here the entries the main text omits. Methods based on backpropagation through time differentiate the critic through the entire sampling chain: besides those cited in \cref{sec:rw}, \citet{lv2025flowbased} constrain the update relative to a behavior policy. GenPO \citep{ding2026genpo} instead applies a PPO-style objective \citep{schulman2017ppo} to a diffusion policy whose likelihood is made tractable by a coupling construction. A second family avoids the chain by regressing onto improved \emph{actions}: DIPO \citep{yang2023dipo} and \citet{li2024learning} improve the actions stored in the replay buffer by gradient ascent on $Q$ and then clone them with the standard denoising loss. This is the closest relative of \method{} among the action-target methods, and the difference is that its regression target is an action rather than a velocity, so the noise level of the interpolant is not accounted for. A third family leaves the denoising loss in place and weights it by the exponentiated critic value \citep{ding2024qvpo,ma2025dpmd,dong2025maxentdp}, of which one-step flow policy mirror descent \citep{chen2025one} is the one-step variant, and flow map policies trained by mirror descent \citep{dong2026mean,gao2026fmer} are the multimodal ones. None of these families obtains its target from the velocity field of the Boltzmann policy itself, which is what the target flow identity provides.

\section{Maximum-entropy reinforcement learning}
\label{app:maxent}
This section provides information on \method{}'s implementation in the online on-policy maximum-entropy RL setting as proposed in \cref{sec:online}.
We review the maximum-entropy actor-critic variant similar to that of \citet{haarnoja2018sac}, where we optimize the state-action value function and the velocity field's parameters.
Moreover, we give an overview of how \method{} resolves bounding final actions that are executed in the environment, provide an algorithm box, report ablations on design choices, and provide detailed per-task results.

\begin{figure}[!tbp]\centering
  \includegraphics[width=\textwidth]{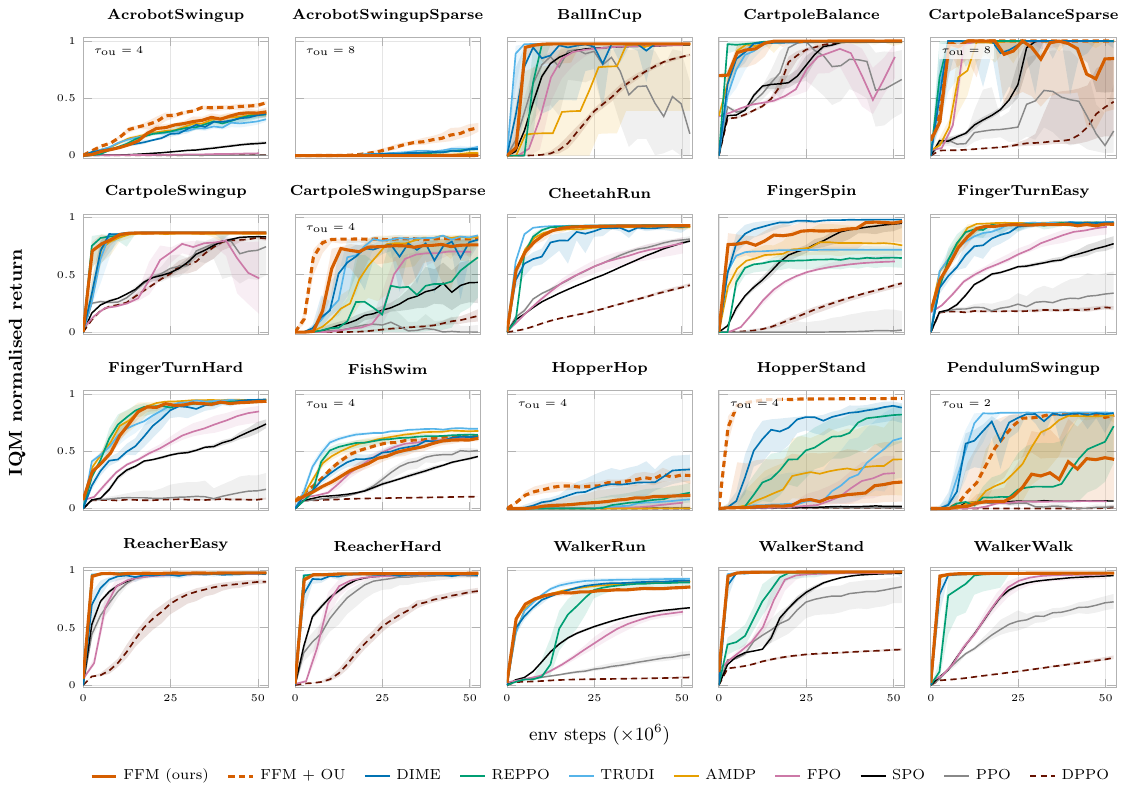}
  \caption{\textbf{Maximum Entropy Online RL: MJX Playground tasks.} We plot the IQM per-task learning curves with $95\%$ stratified bootstrapped confidence intervals \citep{agarwal2021deep} over $10$ seeds. 
  On most of the tasks, \method{} performs similarly to or on par with baselines. 
  However, we observe that there are some tasks such as \textit{AcrobotSwingupSparse, HopperHop, HopperStand}, where \method{} (and most of the baselines) perform badly, indicating that a more sophisticated exploration strategy is required.
  On a subset of the tasks, we have therefore run \method{} with correlated exploration (\method{} + OU), where we show the chosen $\tauou$ hyperparameter on the respective plot. 
  Please note that correlated exploration is only possible due to \method{}'s deterministic mapping (\Cref{sec:online})
  Correlated exploration can help improve performance, such as in \textit{HopperHop}, or even more significantly in the \textit{HopperStand} task. 
  The aggregated results when using the correlated exploration curves whenever applicable are given in \cref{fig:dmc-aggregate} (left two panels).}
  \label{fig:dmc-grid}
\end{figure}

\begin{figure}[!tbp]\centering
  \includegraphics[width=\textwidth]{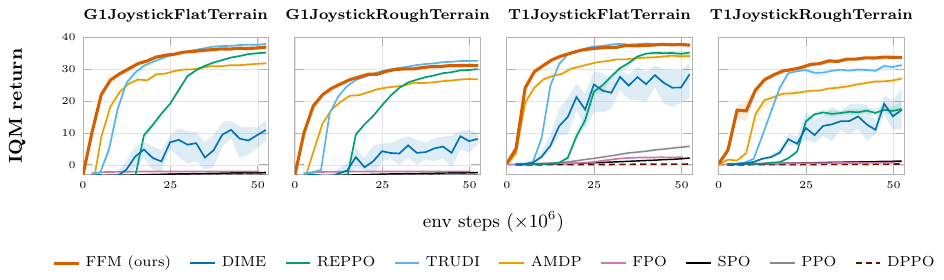}
  \caption{\textbf{Maximum Entropy Online RL: MJX Humanoid tasks.} We plot the IQM per-task learning curves with $95\%$ stratified bootstrapped confidence intervals \citep{agarwal2021deep} over $10$ seeds. \method{} converges faster to a similar performance level as TRUDI and performs significantly better than TRUDI on the \textit{T1JoystickRoughTerrain} task. 
  The high-dimensional humanoid tasks showcase \method{}'s benefits more clearly.
  }
  \label{fig:humanoid-grid}
\end{figure}

\subsection{Actor-Critic based Maximum-entropy reinforcement learning}
\label{app:maxent-prelim}
In maximum-entropy RL \citep{ziebart2008maximum, levine2018rlai, haarnoja2017sql, haarnoja2018sac}, the rewards are augmented with the current policy's entropy, leading to the objective
\begin{equation}
  J(\pi) \;=\; \E_{\pi}\Big[\textstyle\sum_{k\ge 0}\gamma^{k}\big(r(\obs_k,x_k) + \lambda\, \mathcal{H}(\pi(\cdot \mid \obs_k))\big)\Big],
  \label{eq:maxent-objective}
\end{equation}
where $\lambda > 0$ trades off maximizing the expected task return against maximizing the policy's stochasticity by maximizing the entropy $\mathcal{H}(\pi(\cdot \mid \obs_k))$.
The soft state-action value function is given as
\begin{equation}
    Q^{\pi}(\obs,x) = \E_{\pi}\big[\sum_{k\ge 0}\gamma^{k} \big(r(\obs_k,x_k) + \lambda\, \mathcal{H}(\pi(\cdot \mid \obs_k))\big) \mid \obs_0 = \obs,\, x_0 = x\big],
\end{equation}
which can also be expressed using the recursive Bellman equation as
\begin{equation}
  Q^{\pi}(\obs,x) = r(\obs,x) + \gamma\, \E_{\obs' \sim P, x' \sim \pi}\big[Q^{\pi}(\obs', x') - \lambda \log \pi(x'\mid \obs')\big].
  \label{eq:soft-bellman}
\end{equation}
Based on this equation, we can then update the parameters of $Q^{\pi}$ using the Bellman residual, 
\begin{equation}
    J_Q(\phi) = \mathbb{E}\left[\frac{1}{2}\left(Q_\phi^\pi(\obs, x) - y\right)^2\right]
\end{equation}
where $\phi$ are the parameters of the critic network that we have omitted so far to avoid cluttered notation, and $y$ are the critic targets, that are calculated by 
\begin{equation}
    y = r(\obs, x) + \gamma \mathbb{E}_{\obs'\sim p, x'\sim\pi}\left[Q_{\bar{\phi}}^\pi(\obs', x') - \lambda \log\pi(x'\mid \obs')\right]
\end{equation}
for the one-step lookahead, mainly used in the off-policy RL setting \citep{haarnoja2018sac, fujimoto2018td3}.
Here, $\bar{\phi}$ denotes the target network parameters, a Polyak averaging of the actual parameters $\phi$.

In the on-policy setting, commonly no target network is used, and the target is calculated using N-step returns based on a Monte Carlo estimate with samples from the current rollout and the current target as
\begin{equation}
    y = \sum_{k=t}^N \gamma^{k-t}\left(r(\obs_k, x_k) - \lambda \log\pi(x_k\mid\obs_k)\right) + \gamma^{N+1} Q_\phi(\obs_N, x_N),
    \label{eq:N-step-return}
\end{equation}
as suggested, for example, by \citet{reppo2025} and used in our online on-policy methods. Note that the gradients for the critic's parameters $\phi$ are stopped for calculating the targets in \cref{eq:N-step-return}.
For calculating the entropy and the likelihoods of the current policy, we follow \cref{eq:entropy}.
Once the parameters of the Q-function are updated in the policy evaluation step, the parameters of the policy $\pi$ are updated in the policy improvement step such that the expected soft Q-function is maximized.
In the standard case, e.g., with Gaussian policies, the objective for the policy improvement step follows from
\begin{equation}
    L(\pi) = \kl{\pi(x \mid \obs)}{\frac{\exp\big(Q(\obs, x)/\lambda\big)}{\mathcal{Z}(\obs)}},
\end{equation}
where $\mathcal{Z}(\obs)=\int \exp\big(Q(\obs, x)/\lambda\big)\,\dt x$ is the normalizing constant and independent of the action $x$.
It can be shown that alternating both steps converges to the optimal maximum-entropy policy in the tabular setting \citep{haarnoja2017sql}.
For \method{} to optimize the flow policy, we rely on \cref{eq:tr_ffm}.

\paragraph{Enforcing action bounds}
In RL benchmarks and potentially in real-world use cases, the applied actions are bounded due to physical constraints such as bounded motor torques.
\method{} squashes the generated actions to $[-1,1]^d$, to avoid querying the Q-function at unreasonable action regions.
However, please note that we keep the flow unconstrained on $\R^d$, such that the theory does not change, but we map the flow's clean output $X_1$ to the environment action through a coordinate-wise squashing function $g$
\begin{equation}
  x_{\mathrm{env}} = g(X_1),
\end{equation}
where, in the literature \citep{haarnoja2018sac}, the $\tanh$ function is used as
\begin{equation}
    g(z) = \tanh(z),
\end{equation}
with the log derivative
\begin{equation}
    \log g'_{\tanh}(z) = \log\big(1 - \tanh^2 z\big).
    \label{eq:tanh-squash-jac}
\end{equation}
Recently, \citet{amdp2025} proposed an alternative squashing function and reported that using the error function
\begin{equation}
    g(z) = \mathrm{erf}(z/\sqrt2) 
\end{equation}
yields improved numerical behavior.
Its log derivative is given by
\begin{equation}
    \log g'_{\mathrm{erf}}(z) = \tfrac12\log\tfrac{2}{\pi} - \tfrac{z^2}{2} .
  \label{eq:erf-squash-jac}
\end{equation}

However, please note that squashing the actions changes the likelihood calculations by the change-of-variables rule, such that the likelihood of the new distribution is
\begin{equation}
    \pi_{\mathrm{env}}(\cdot\mid\obs) = \pi(\cdot\mid\obs) \Biggl|\mathrm{det}\left(\frac{dx_{\mathrm{env}}}{dx}\right)\Biggr|^{-1},
    \label{eq:CoV likelihood}
\end{equation}
which leads to the log-likelihoods needed for calculating the entropy values as
\begin{equation}
    \log \pi_{\mathrm{env}}(\cdot\mid\obs) = \log \pi(\cdot\mid\obs) - \sum_i \log g'(X_{1,i}),
\end{equation}
where, depending on which squashing function is used, either \cref{eq:tanh-squash-jac} or \cref{eq:erf-squash-jac} is used for the Jacobian in \cref{eq:CoV likelihood}.
This log-likelihood is required to estimate the entropy (\cref{eq:entropy}) over a batch of samples during the rollout to calculate the soft Q-function in maximum-entropy RL.

Because the \method{} target from \cref{eq:tfm-objective} is w.r.t. the clean action $X_1$ rather than the squashed action $x_{\mathrm{env}}$, a consequence of squashing the clean actions is that, by the chain rule, the critic's gradient w.r.t. ${X_1}$ is
\begin{equation}
    \nabla_{X_1} Q(\obs, g(X_1)) = g'(X_1) \odot \nabla_{x_{\mathrm{env}}} Q(\obs, x_{\mathrm{env}}),
\end{equation}
where $\odot$ denotes coordinate-wise multiplication.
In our experiments, we have used the error function as the squashing function, although, as shown in our ablation studies, the tanh squashing only performs slightly worse (\cref{app:online-results}).

\subsection{Pseudocode}
\label{app:algorithm}

\begin{algorithm}[H]
\caption{Online on-policy maximum-entropy RL with an \method{} actor}
\label{alg:maxent-online}
\begin{algorithmic}[1]
\State \textbf{input:} Initialized parameters for velocity field $\vel_\theta$, critic $Q_\phi$, temperature $\lambda$ with entropy target $\mathcal{H}^\star$, trust-region budget $\varepsilon_{\mathrm{TR}}$ with multiplier $\nu$, schedules $\alpha, \beta$ (\cref{eq:geom-schedule}), time weighting $w$
\For{each iteration}
  \State \textbf{rollout (on-policy, no gradients):}
  \Statex \hspace{\algorithmicindent}\hspace{\algorithmicindent} \textbullet\ act $x_k \sim \pi_\theta(\cdot \mid \obs_k)$ by integrating \cref{eq:ode}
  \Statex \hspace{\algorithmicindent}\hspace{\algorithmicindent} \textbullet\ optionally, correlate the prior draws via the OU chain of \cref{eq:ou}
  \Statex \hspace{\algorithmicindent}\hspace{\algorithmicindent} \textbullet\ collect $\{(\obs_k, x_k, r_k, \obs_{k+1})\}_{k=t}^{N}$
  \State \textbf{likelihoods:} $\log \pi_\theta(x_k \mid \obs_k)$ by \cref{eq:entropy}, divergences by Hutchinson's estimator
  \State \textbf{policy evaluation:} calculate $N$-step soft returns of \cref{eq:N-step-return} (target gradients stopped)
  \Statex \hspace{\algorithmicindent}\hspace{\algorithmicindent} $\displaystyle y_t = \sum_{k=t}^{N} \gamma^{\,k-t}\bigl(r(\obs_k, x_k) - \lambda \log \pi_\theta(x_k \mid \obs_k)\bigr) + \gamma^{\,N+1} Q_\phi(\obs_N, x_N)$
  \For{each critic epoch over minibatches of the rollout}
    \State update critic parameters $\phi$ using $\bigl(Q_\phi(\obs_t, x_t) - y_t\bigr)^2$, or if distributional using HL-Gauss \citep{pmlr-v235-farebrother24a, reppo2025}
  \EndFor
  \State \textbf{Calculate Q-scores:} $\nabla_x Q_\phi(\obs, X_1)$ at the rollout actions $X_1 \sim \pi_{\bar\theta}(\cdot \mid \obs)$
  \For{each actor epoch over minibatches of the rollout }
    \State $t \sim \mathcal{U}(0,1)$, $\varepsilon \sim \gauss{\cdot}{0}{I}$, $X_t \gets \alpha(t) X_1 + \beta(t)\varepsilon$
    \State minimize over $\theta$ using \cref{eq:lagrangian}:
    \Statex \hspace{\algorithmicindent}\hspace{\algorithmicindent} $\displaystyle
    \mathbb{E}_{t, X_1, \varepsilon}\Bigl[ w(t) \Bigl\| \vel_\theta(X_t, t \mid \obs) - \tfrac{\dot\alpha(t)}{\alpha(t)} X_t - \tfrac{\cs(t)}{\lambda\,\alpha(t)} \nabla_x Q_\phi(\obs, X_1) \Bigr\|^2 \Bigr]
    + \nu \bigl( \mathcal{L}_{\mathrm{TR}}(\theta, \bar\theta) - \varepsilon_{\mathrm{TR}} \bigr)$
    \Statex \hspace{\algorithmicindent}\hspace{\algorithmicindent} with $\mathcal{L}_{\mathrm{TR}}$ of \cref{eq:klflow} evaluated at the same points $(X_t, t)$
    \State \textbf{duals:} 
    \Statex \hspace{\algorithmicindent}\hspace{\algorithmicindent} \textbullet\ ascend $\nu$ (\cref{eq:tr_minmax})
    \Statex \hspace{\algorithmicindent}\hspace{\algorithmicindent} \textbullet\ descend $\lambda$ on $\lambda\bigl(\mathbb{E}\bigl[\widehat{\mathcal{H}}\bigr] - \mathcal{H}^\star\bigr)$, so $\lambda$ grows while the entropy is below $\mathcal{H}^\star$
  \EndFor
  \State \textbf{update the trust-region reference:} $\bar\theta \gets \theta$
\EndFor
\State \textbf{evaluation:} integrate \cref{eq:ode} from $X_0 = 0$ (\cref{app:online-results})
\end{algorithmic}
\end{algorithm}

\subsection{Experimental details}
\label{app:online-exp}

\paragraph{Environments and tasks.}
In our online on-policy RL experiments from \cref{sec:experiments} we have evaluated \method{} on the GPU-accelerated MuJoCo Playground \citep{mujoco_playground_2025} tasks.
More concretely, we distinguish between the 20 DeepMind Control (DMC) tasks of the suite, whose episode returns are bounded between $[0, 1000]$, and the four MJX humanoid joystick tasks, which require precise and high-dimensional whole-body control and whose returns are not normalized.

For every learning curve, we follow \citet{agarwal2021deep} and report the interquartile mean (IQM) of the return with stratified bootstrapped $95\%$ confidence intervals over 10 seeds per task.
Furthermore, in \cref{fig:dmc-aggregate} the aggregated results over all tasks for the Playground and the Humanoid tasks are reported, where the per-task results are shown in \cref{app:online-results}.

\paragraph{Baselines.}
We compare against Gaussian-policy methods, REPPO \citep{reppo2025}, PPO \citep{schulman2017ppo}, and SPO \citep{xie2025simple}, diffusion-policy methods, TRUDI \citep{trudi2025}, AMDP \citep{amdp2025}, DIME \citep{celik2025dime}, and DPPO \citep{ren2025dppo}, and the flow-policy method FPO \citep{mcallister2026flow}.

\paragraph{Hyperparameters.}
\Cref{tab:hyperparameters} lists the hyperparameters of \method{} and the baselines. 
Notably, for \method{} the flow starts from the prior $X_0 \sim \gauss{\cdot}{0}{\sigma_0^2 I}$ with $\sigma_0 = 1$ under the geometric schedule of \cref{eq:geom-schedule} and is integrated with $16$ Euler steps. 
Furthermore, \method{} employs action squashing using the error function map from \cref{app:maxent-prelim} and the entropy estimate \cref{eq:entropy} from \cref{sec:online} with four Rademacher probes per integration node.

\paragraph{Implementation details.}

\emph{Distributional critic.} As done in most baselines and related works recently, we parameterize the critic distributionally following HL-Gauss \citep{pmlr-v235-farebrother24a}, trained by cross-entropy on the same targets as \cref{eq:N-step-return}.

\emph{Actor regression.} The implemented policy-improvement step is the \method{} regression \cref{eq:fp-tfm-loss} with the regression target replaced by the variance-optimal mixture of \cref{eq:velmix} and the time weighting fixed to $w(t) = 1/\cs^2(t)$,
\begin{equation}
  \mathcal{L}(\theta, \bar\theta)
  \;=\; \Eintt{\pi_{\bar\theta}}\Big[ w(t)\,
  \big\| \vel_\theta(X_t, t \mid \obs) - Y_{c^\star}(X_1, \varepsilon, t)
  \big\|^2 \Big],
  \qquad
  w(t) \;=\; \frac{1}{\cs^2(t)},
  \label{eq:impl-actor}
\end{equation}
where $Y_{c^\star}$ is the velocity target \cref{eq:velmix} at the variance-optimal mixing schedule of \cref{eq:velmix-cstar}. 
The variance-optimal mixing schedule is given as $c^\star(t) = \alpha^2(t)/\big(\alpha^2(t) + \beta^2(t)\,\bar{\mathcal I}\big)$ (\cref{eq:velmix-cstar}), but we fixed the Fisher scale to $\bar{\mathcal I} = 1$, such that we obtain $c^\star(t) = \alpha^2(t)/\big(\alpha^2(t) + \beta^2(t)\big)$.
Since the generated actions are squashed (\cref{app:maxent-prelim}), the Q-score $\nabla_x Q(\obs, X_1)/\lambda$ in $Y_{c^\star}$ becomes
$\nabla_x \big[Q(\obs, g(x))/\lambda
+ \log|{\det \partial g(x)}|\big]\big|_{x = X_1}$, which is the score of the squashed targets. 
For $c^\star = 1$, \cref{eq:impl-actor} is \cref{eq:fp-tfm-loss}.
Please note that multiplying the residual by $\sqrt{w(t)} = 1/\cs(t)$ converts the velocity
difference into a score difference through the bridge \cref{eq:app-bridge}, so \cref{eq:impl-actor} is the score-field regression of \cref{app:targets}. The regression \cref{eq:impl-actor} in its score-field form is 
\begin{multline}
  \mathcal{L}(\theta, \bar\theta)
  \;=\; \Eintt{\pi_{\bar\theta}}\bigg[\,
  \Big\| \frac{\alpha(t)\,\bigl(\vel_\theta(X_t, t \mid \obs)
  - \dot\alpha(t)\, X_1\bigr) + a(t)\,\varepsilon}{\alpha(t)\,\cs(t)}
  \\
  - \frac{c^\star(t)}{\alpha(t)}\,
  \nabla_x \Big[\tfrac{1}{\lambda} Q\big(\obs, g(x)\big)
  + \log|{\det \partial g(x)}|\Big]\Big|_{x = X_1} \Big\|^2 \,\bigg],
  \label{eq:impl-scorefield}
\end{multline}
with the coefficient $a(t) := -\big( c^\star \dot\alpha\beta + (1 - c^\star)\,\alpha\dot\beta \big)$.
However, since $1/\cs(t)$ diverges as $t \to 1$ for the linear schedule and $1/\alpha(t)$ diverges as $t \to 0$, the implementation softens the two denominators of \cref{eq:impl-scorefield}, $\alpha\cs \to \alpha\cs +
\delta\sigma_0^2$ and $\alpha \to \alpha + \delta$, by the constant
$\delta = 0.05$.

\emph{Deterministic evaluation.} During evaluation, we found that starting the flow's integration from the central point $X_0 = 0$ instead of a random prior draw yields improved performance (\cref{app:online-results}); training rollouts keep stochastic prior sampling.

\paragraph{Interpolant schedule.}
We instantiate the schedules $\alpha,\beta$ with a \emph{geometric}
schedule. Given interpolation scales $\sigma_{\max}>\sigma_{\min}>0$ with
ratio $r=\sigma_{\max}/\sigma_{\min}$, we write
$\tau(t)=\sigma_{\max}^2\bigl(1-r^{-2t}\bigr)$ and
$\tau_T=\sigma_{\max}^2-\sigma_{\min}^2$, and set
\begin{equation}
  \alpha(t)\;=\;\frac{\tau(t)}{\tau_T}\;=\;\frac{1-r^{-2t}}{1-r^{-2}},
  \qquad
  \beta(t)\;=\;\sqrt{\tau_T-\tau(t)}\;=\;\sigma_{\max}\sqrt{r^{-2t}-r^{-2}}.
  \label{eq:geom-schedule}
\end{equation}
The boundary conditions of the one-sided interpolant hold exactly since
$\tau(0)=0$ and $\tau(1)=\sigma_{\max}^2(1-r^{-2})=\tau_T$, which gives
$\alpha(0)=0$, $\alpha(1)=1$, $\beta(1)=0$, and a prior scale of
$\sigma_0=\beta(0)=\sqrt{\smash[b]{\tau_T}}$.
In our
experiments the geometric schedule led to slightly better results
(\cref{fig:schedule-comparison,fig:schedule-dmc-grid,fig:schedule-humanoid-grid}) than the
linear schedule $\alpha(t)=t$, $\beta(t)=\sigma_0(1-t)$. All runs
use $\sigma_0=1$ and $r=2.5$, corresponding to
$\sigma_{\max}=\sigma_0/\sqrt{1-r^{-2}}\approx1.09$ and
$\sigma_{\min}=\sigma_{\max}/r\approx0.44$, identical across all task
suites.

\begin{table}[!htbp]
\centering
\caption{Hyperparameters of \method{} and of the online baselines.}
\label{tab:hyperparameters}
\resizebox{\textwidth}{!}{%
\begin{tabular}{lcccccccc}
\toprule
Parameter & \shortstack{\method{}\\(ours)} & AMDP & REPPO & DIME & SPO & PPO & FPO & DPPO \\
\midrule
\multicolumn{9}{l}{\textit{Actor Network}} \\
Hidden dim.       & 512  & 512  & 512  & 512  & 256  & 256  & 32   & 256  \\
Hidden layers     & 3    & 3    & 3    & 3    & 3    & 3    & 5    & 3    \\
Activation        & GeLU & GeLU & GeLU & GeLU & ELU  & ELU  & SiLU & Mish \\
Flow/diff.\ steps & 16   & 16   & N/A  & 16   & N/A  & N/A  & 10   & 8    \\
\midrule
\multicolumn{9}{l}{\textit{Critic Network}} \\
Critic type   & \shortstack{Dist.\\(HL-Gauss)} & \shortstack{Dist.\\(HL-Gauss)}
              & \shortstack{Dist.\\(HL-Gauss)} & N/A & MSE & MSE & MSE & MSE \\
Hidden dim.   & 512      & 512      & 512      & 512 & 256 & 256 & 256 & 256 \\
Hidden layers & 3        & 3        & 3        & 3   & 3   & 3   & 5   & 3   \\
Ensemble/bins & 151 bins & 151 bins & 151 bins & 100 & 1   & 1   & 1   & 1   \\
Activation    & GeLU     & GeLU     & GeLU     & GeLU & ELU & ELU & SiLU & Mish \\
\midrule
\multicolumn{9}{l}{\textit{Optimization}} \\
Optimizer      & Adam & Adam & Adam & Adam & Adam & Adam & Adam & Adam \\
Actor LR       & $3{\times}10^{-4}$ & $3{\times}10^{-4}$ & $3{\times}10^{-4}$
               & $3{\times}10^{-4}$ & $1{\times}10^{-3}$ & $1{\times}10^{-3}$
               & $3{\times}10^{-4}$ & $3{\times}10^{-4}$ \\
Critic LR      & $3{\times}10^{-4}$ & $3{\times}10^{-4}$ & $3{\times}10^{-4}$
               & $3{\times}10^{-4}$ & $1{\times}10^{-3}$ & $1{\times}10^{-3}$
               & $3{\times}10^{-4}$ & $3{\times}10^{-4}$ \\
Max grad.\ norm & 0.5 & 0.5 & 0.5 & N/A & 1.0 & 1.0 & 1.0 & 1.0 \\
Constraint     & \shortstack{$\epsilon=0.1$\\KL} & \shortstack{$\epsilon=0.1$\\KL}
               & \shortstack{$\epsilon=0.1$\\KL} & N/A
               & \shortstack{$\epsilon=0.2$\\Penalty} & \shortstack{$\epsilon=0.2$\\Clip}
               & \shortstack{$0.2$\\Clip} & \shortstack{$0.2$\\Clip} \\
\midrule
\multicolumn{9}{l}{\textit{Other Settings}} \\
Discount$^{\dagger}$ & $0.99$\,/\,$0.97$ & $0.99$\,/\,$0.97$ & $0.99$\,/\,$0.97$ & $0.99$\,/\,$0.97$ & $0.99$\,/\,$0.97$ & $0.99$\,/\,$0.97$ & $0.99$\,/\,$0.97$ & $0.99$\,/\,$0.97$ \\
Batch size     & 1024   & 1024   & 1024   & 1024    & 1024 & N/A & N/A & N/A \\
Buffer size    & 131072 & 131072 & 131072 & 131072 & N/A  & N/A & N/A & N/A \\
Target entropy & $-0.1\,\mathrm{dim}(\mathcal A)$ & $-1.0\,\mathrm{dim}(\mathcal A)$
               & $-0.5\,\mathrm{dim}(\mathcal A)$ & $-4\,\mathrm{dim}(\mathcal A)$
               & N/A & N/A & N/A & N/A \\
\bottomrule
\end{tabular}}%
\par\smallskip
{\raggedright\footnotesize $^{\dagger}$\,All methods use discount $0.99$ on the
MuJoCo Playground DMC tasks and $0.97$ on the humanoid locomotion tasks
(G1/T1 Joystick); all other \method{} hyperparameters are identical across the two
suites.\par}
\end{table}

\subsection{Further experiments and design ablations}
\label{app:online-results}
\Cref{fig:dmc-grid} and \cref{fig:humanoid-grid} show the per-task learning curves on the 20 DMC tasks and the four humanoid tasks, respectively.
An aggregation over all environments can be found in \cref{fig:dmc-aggregate}.

\paragraph{Interpolant schedule.} As detailed in \Cref{app:online-exp}, we use a geometric interpolant schedule for the on-policy RL setting. 
We refer to \Cref{app:online-exp} for the exact parameter setting. 
A valid and intuitive interpolant is the linear schedule as described in \Cref{sec:prelim}. 
However, the schedule type's effect on performance is unclear, which is why we have conducted thorough experiments and report the aggregated results over the 20 Playground tasks together with their final performances and the Humanoid tasks together with their final performances in \cref{fig:schedule-comparison}. 
In \Cref{fig:schedule-dmc-grid}, we report the per-task results of the Playground suite, and the per-task results on the humanoid environments are shown in \Cref{fig:schedule-humanoid-grid}.
While there seems to be only a minor, negligible difference in the aggregated results for the Playground tasks (\Cref{fig:schedule-comparison} (left)), a clear difference can be seen in favor of the geometric schedule in the aggregated results for the humanoid tasks (\Cref{fig:schedule-comparison} (right)). 
However, our experiments show that both schedules lead to a reasonably well-performing policy on average. 

\begin{figure}[!htbp]\centering
  \includegraphics[width=\textwidth]{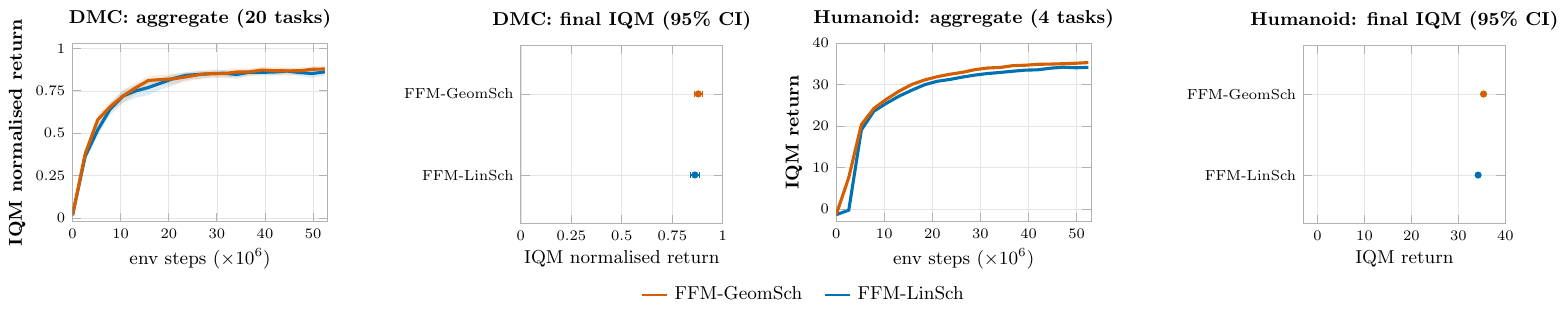}
  \caption{\textbf{Ablation. Geometric versus linear interpolant schedule. } As discussed in \cref{app:online-results}, \method{} is not restricted to a specific interpolation schedule type such as the linear schedule. 
  Motivated by diffusion-based methods, we have compared the performance of a linear schedule to the geometric schedule as described in \cref{app:online-results} and report the aggregated results over all 20 MJX Playground tasks (left two panels) and the four humanoid tasks (right two panels), where every single task was run on $10$ seeds, and we report the IQM and the stratified bootstrapped $95\%$ confidence intervals \citep{agarwal2021deep}.
  While the results indicate that there is barely a performance difference in the Playground tasks, a slight improvement on the humanoid tasks can be seen, which is why we have chosen the geometric schedule as default for the online RL tasks. See \cref{fig:schedule-dmc-grid} and \cref{fig:schedule-humanoid-grid} for the per-task results.}
  \label{fig:schedule-comparison}
\end{figure}

\begin{figure}[!tbp]\centering
  \includegraphics[width=\textwidth]{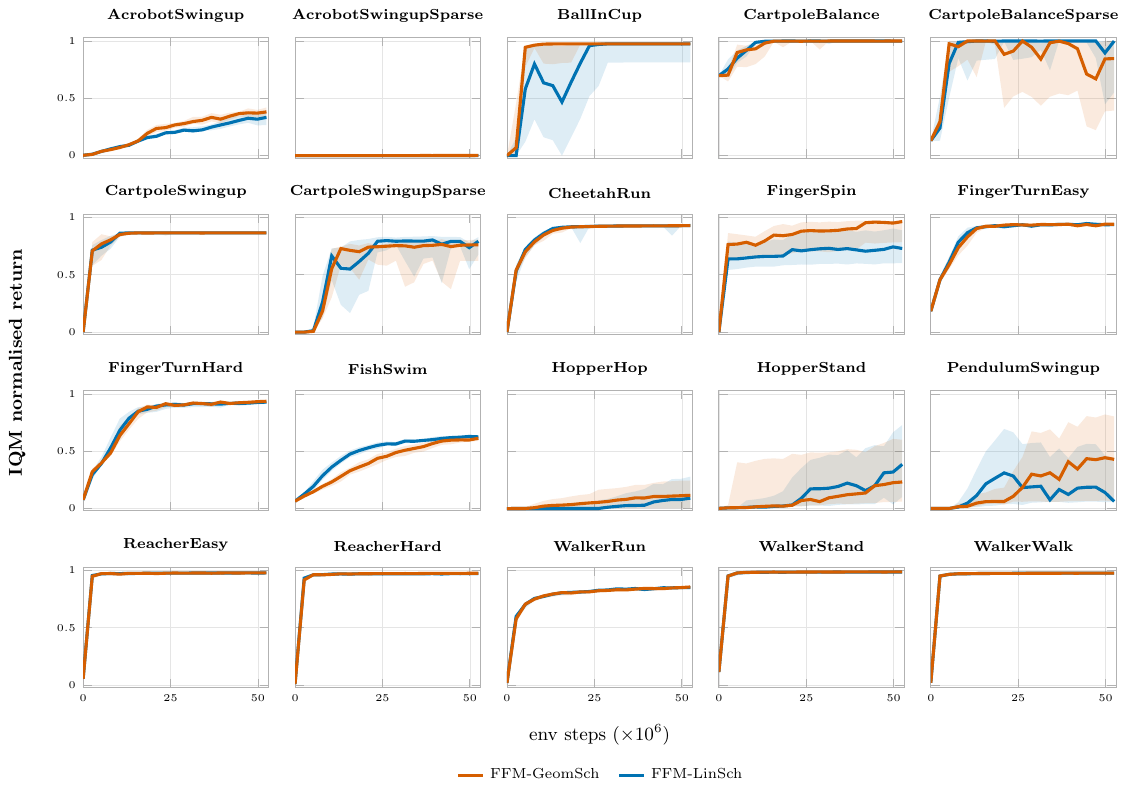}
  \caption{\textbf{Ablation. Geometric versus linear interpolant schedule (MJX Playground).} We plot the IQM per-task learning curves with $95\%$ stratified bootstrapped confidence intervals \citep{agarwal2021deep} over $10$ seeds on the 20 MJX Playground tasks. 
  The performance between the schedule types is the same for most tasks and differs only for some tasks such as \textit{FingerSpin} or \textit{PendulumSwingup}, which, however, lie within the confidence intervals. 
  Therefore, there is no statistically clear result on which interpolant should be used for the Playground tasks.} 
  \label{fig:schedule-dmc-grid}
\end{figure}

\begin{figure}[!tbp]\centering
  \includegraphics[width=\textwidth]{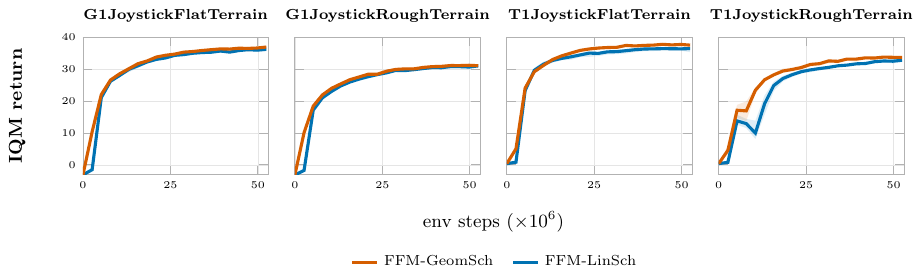}
  \caption{\textbf{Ablation. Geometric versus linear interpolant schedule (MJX Humanoid).} We plot the IQM per-task learning curves with $95\%$ stratified bootstrapped confidence intervals \citep{agarwal2021deep} over $10$ seeds on the four MJX humanoid tasks. 
  While for the \textit{G1JoystickFlatTerrain} and for the \textit{G1JoystickRoughTerrain} tasks there is no difference in performance, it seems that the geometric schedule performs slightly better on the \textit{T1} tasks, which is why we have chosen the geometric schedule as default for the online RL experiments. 
  }
  \label{fig:schedule-humanoid-grid}
\end{figure}

\paragraph{Number of integration steps.} We further analyze the effect of the number of integration steps on the humanoid tasks' performance. 
\Cref{fig:online-ablations} (left) shows the results aggregated over the four MJX humanoid
tasks for different numbers of Euler integration steps $K$. 
Please note that we only vary $K$ and keep everything else as detailed in \Cref{tab:hyperparameters}. The corresponding per-task curves to the aggregated results are shown in \Cref{fig:step-ablation-grid}.
As expected, the performance increases with the number of integration steps, since the integration
error of the flow decreases. 
However, we observe that performance plateaus at $K=16$, where increasing to $K=32$ does not improve performance and yields similar results. 
Based on this observation and the baseline's hyperparameters, we use $K = 16$ throughout the paper for a fair comparison against the other diffusion- and flow-based methods (\cref{tab:hyperparameters}).

\begin{figure}[!htbp]\centering
  \includegraphics[width=0.508\textwidth]{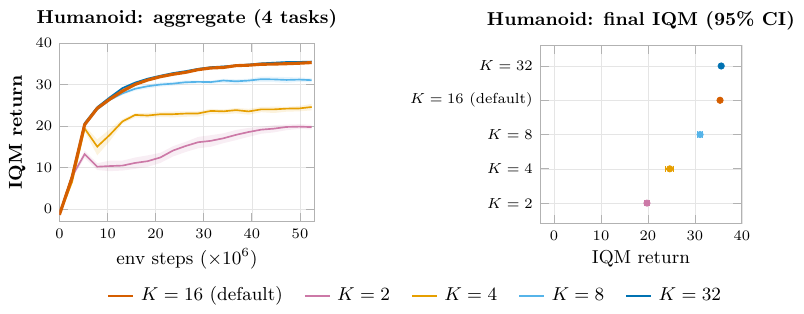}\hfill
  \includegraphics[width=0.486\textwidth]{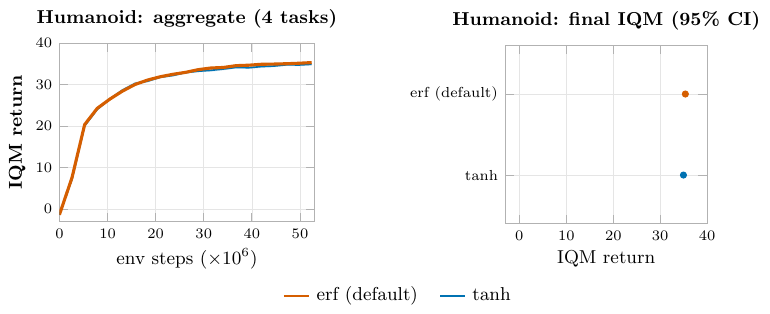}
  \caption{\textbf{Ablation. Number of Integration Steps (left two panels) and Action Squashing mapping (right two panels).} \textit{Number of integration steps (left two panels):} We analyze the effect of varying integration steps and report the aggregated results over all four humanoid tasks (left two panels). 
  Every single task was run on $10$ seeds, and we report the IQM and the $95\%$ stratified bootstrapped confidence intervals \citep{agarwal2021deep}.
  The results indicate that with an increasing number of integration steps, the performance increases until it plateaus from $K=16$ onwards, which is why we have chosen $K=16$ as default in our experiments. 
  Please see \cref{fig:step-ablation-grid} for the per-task learning curves.
  \textit{Action Squashing Mapping (right two panels):} We analyze the effect of utilizing the $\mathrm{tanh}$ and the $\mathrm{erf} $ as squashing functions for mapping the generated actions to a defined action bound. 
  We have run experiments on the four MJX humanoid tasks and report the aggregated IQM with $95\%$  stratified bootstrapped confidence intervals ($10$ seeds per task) \citep{agarwal2021deep} in the two right panels.
  The results show that \method{} is not sensitive to the choice of the squashing function. 
  Please see \cref{fig:squash-ablation-grid} for the per-task learning curves.}
  \label{fig:online-ablations}
\end{figure}

\begin{figure}[!htbp]\centering
  \includegraphics[width=\textwidth]{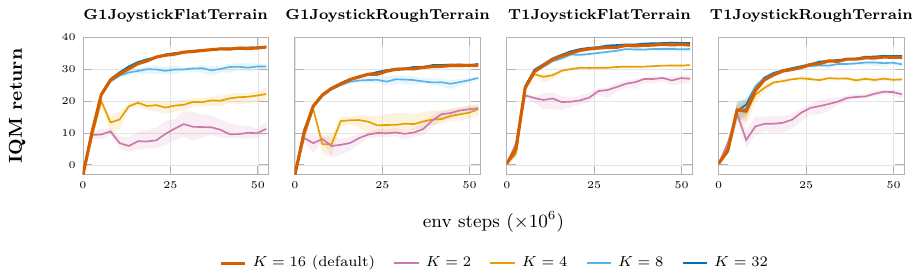}
  \caption{\textbf{Ablation. Number of Integration Steps.} We plot the IQM per-task learning curves with $95\%$ stratified bootstrapped confidence intervals \citep{agarwal2021deep} over $10$ seeds on the four MJX humanoid tasks for varying numbers of integration steps $K$. 
  With increasing number of integration steps, the performance improves. 
  However, the performance plateaus at $K=16$ steps.}
  \label{fig:step-ablation-grid}
\end{figure}

\paragraph{Action squashing.} A key factor for stable training in actor-critic-based RL is that the generated actions need to be squashed to a valid region to avoid querying the critic at unreasonable action regions during the policy update. 
As detailed in \cref{app:maxent}, there are two distinctive choices for the squashing function, namely, the $\mathrm{tanh}$ and the $\mathrm{erf}$ functions. 
To assess the effect on the performance of the design choices, we have run \method{} with a $\mathrm{tanh}$ squashing function while keeping all other design choices the same and report the aggregated learning curves on the humanoid tasks in \cref{fig:online-ablations}(right), together with the final returns. 
The per-task results are shown in \cref{fig:squash-ablation-grid}. 
From the results, it becomes clear that both variants perform on par, indicating that \method{} is not sensitive to this choice and that either function is a valid default. 
\begin{figure}[!htbp]\centering
  \includegraphics[width=\textwidth]{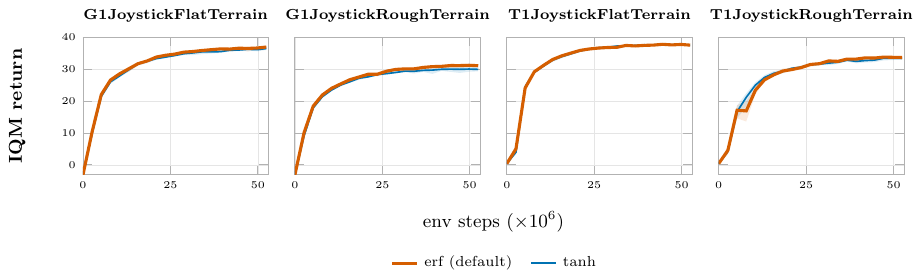}
  \caption{\textbf{Ablation. Action squashing mapping. } We plot the IQM per-task learning curves with $95\%$ stratified bootstrapped confidence intervals \citep{agarwal2021deep} over $10$ seeds on the four MJX humanoid tasks for the $\mathrm{tanh}$ and $\mathrm{erf}$ squashing functions. 
  The results indicate that \method{} is insensitive to the choice of the squashing function.
}
  \label{fig:squash-ablation-grid}
\end{figure}

\paragraph{Deterministic evaluation.} During our experiments, we observed that starting the flow's integration always at the central point $X_0=0$ during evaluation only, yields improved performance against the standard case, where $X_0\sim \mathcal{N}(0, I)$ is drawn from a standard normal distribution as done during training. 
This observation is depicted in \cref{fig:stoch-eval}, where we plot the aggregated learning curves on the Playground tasks together with their final returns (left) and the aggregated learning curves on the humanoid tasks together with their final returns (right). 
The per-task learning curves are shown in \cref{fig:stoch-eval-dmc-grid} and
\cref{fig:stoch-eval-humanoid-grid} respectively. 
While the improvement is not very obvious for the Playground tasks (\cref{fig:stoch-eval} (left)), a clear and significant performance boost in the humanoid tasks can be observed (\cref{fig:stoch-eval} (right)). 
This result indicates that, for high-dimensional tasks in particular, removing the stochasticity from the action generation process yields improved performance. 
While this observation is very interesting and worth a more detailed analysis, in particular from a theoretical point of view, we leave this analysis for future work and apply the deterministic evaluation throughout our experiments. 

\begin{figure}[!htbp]\centering
  \includegraphics[width=\textwidth]{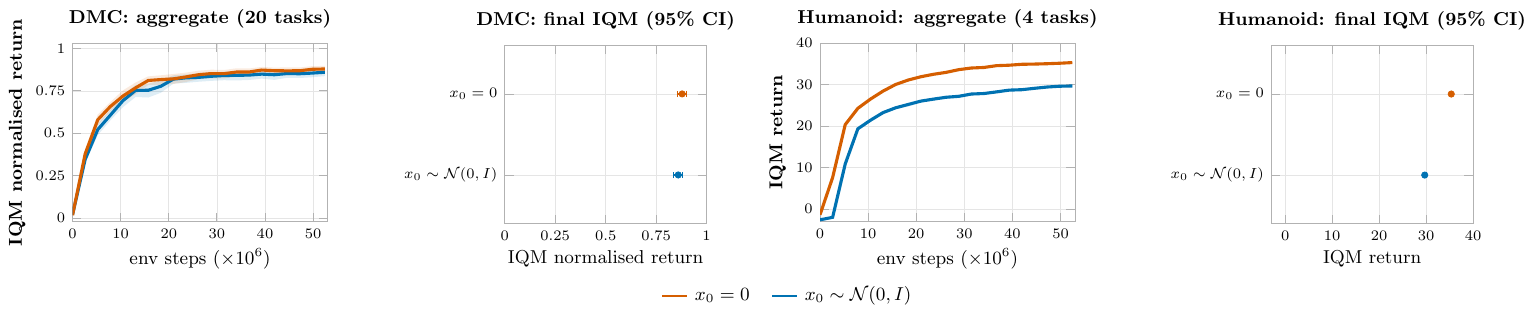}
  \caption{\textbf{Ablation. Deterministic evaluation.} We compare the effect of evaluating \method{} using $X_0 = 0$ (deterministic) and $X_0 \sim \mathcal{N}(0, I)$ (stochastic) on the 20 MJX Playground (left two panels) and four humanoid tasks (right two panels), where every task was run on $10$ seeds, and we plot the aggregated IQM with stratified bootstrapped $95\%$ confidence intervals \citep{agarwal2021deep}. 
  While there is almost no performance gap between both variants for the Playground tasks, a significant performance boost follows from the deterministic evaluation on the humanoid tasks. 
  This indicates that deterministic evaluation matters most for high-dimensional tasks. 
  Please see \cref{fig:stoch-eval-dmc-grid} and \cref{fig:stoch-eval-humanoid-grid} for the per-task performances of each task suite.
}
  \label{fig:stoch-eval}
\end{figure}

\begin{figure}[!htbp]\centering
  \includegraphics[width=\textwidth]{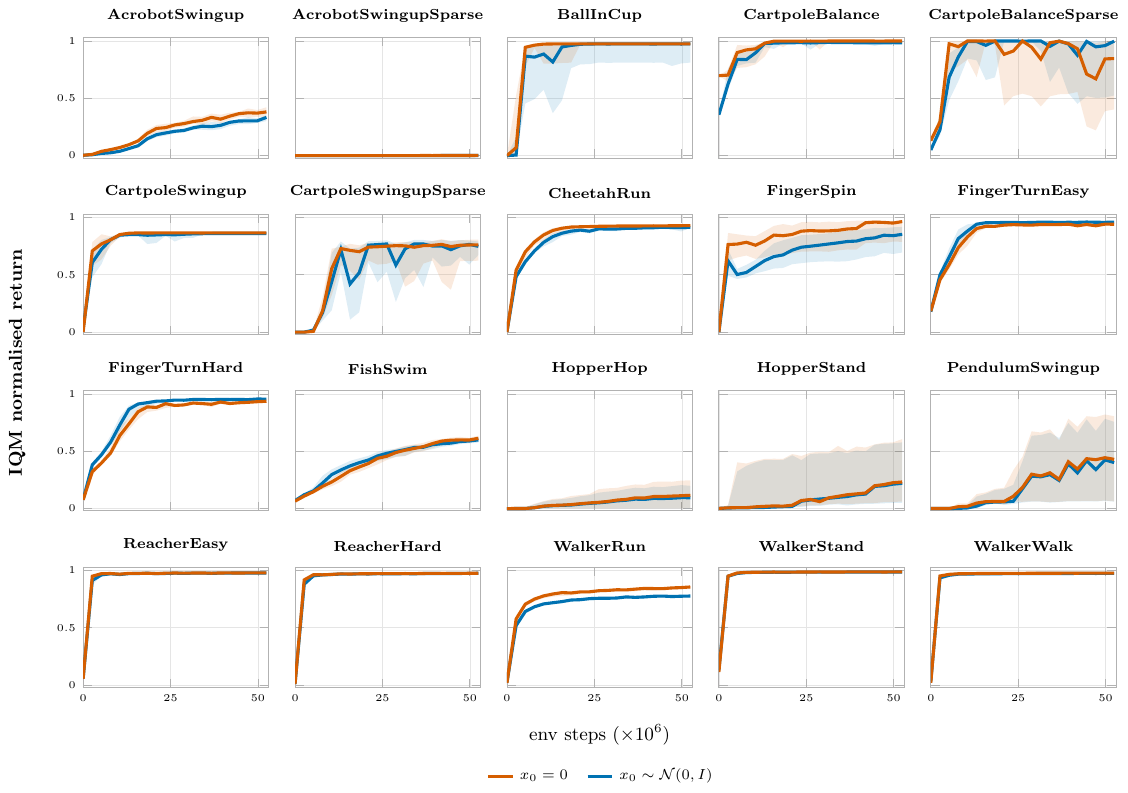}
  \caption{\textbf{Ablation. Deterministic evaluation (MJX Playground).} We plot the IQM per-task learning curves with $95\%$ stratified bootstrapped confidence intervals \citep{agarwal2021deep} over $10$ seeds on each of the 20 MJX Playground tasks for $X_0=0$ (deterministic) and $X_0 \sim \mathcal{N}(0, I)$ (stochastic) during evaluation.
  While most of the results indicate that there are not many differences, for \textit{FingerSpin}, or \textit{WalkerRun}, deterministic evaluation seems to improve performance. }
  \label{fig:stoch-eval-dmc-grid}
\end{figure}

\begin{figure}[!htbp]\centering
  \includegraphics[width=\textwidth]{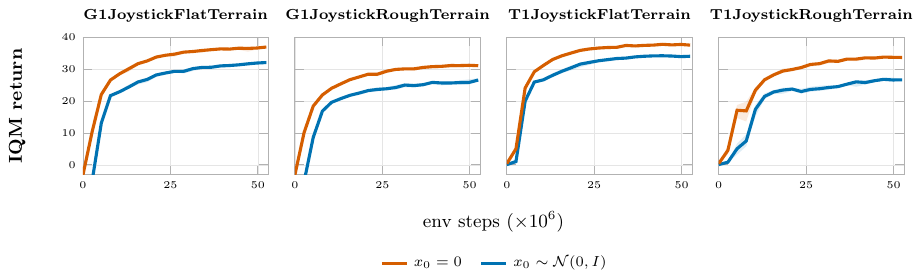}
  \caption{\textbf{Ablation. Deterministic evaluation (MJX Humanoids).} We plot the IQM per-task learning curves with $95\%$ stratified bootstrapped confidence intervals \citep{agarwal2021deep} over $10$ seeds on each of the four MJX humanoid tasks for $X_0=0$ (deterministic) and $X_0 \sim \mathcal{N}(0, I)$ (stochastic) during evaluation.
  For all four tasks, deterministic evaluation shows significantly improved performance, indicating that stochasticity can harm performance, in particular for high-dimensional action spaces. We use deterministic evaluation throughout all of our experiments. }
  \label{fig:stoch-eval-humanoid-grid}
\end{figure}

\paragraph{Correlated exploration.} A flow policy's key benefit is that the action generation process is a deterministic mapping up to the sampling of the starting sample. 
As discussed in \cref{sec:online}, this allows us to apply a correlated exploration scheme, where the exploration is shaped by correlating the prior draws across the environment steps. 
To analyze the additional hyperparameter $\tauou$ that is introduced by \cref{eq:ou}, we have conducted experiments on single Playground tasks, where the initial instantiation of \method{}, i.e., without correlated exploration, performs unsatisfactorily (see \cref{fig:dmc-grid}). 
\Cref{fig:ou-grid} shows the learning curves on a subset of tasks that we have picked from the task suite to analyze the effect on performance for varying $\tauou$ values. 
We can observe that for each task, there exists at least one $\tauou$ value that outperforms the initial case without correlated exploration, and only for the \textit{FingerSpin} and \textit{WalkerRun} tasks, we observe that a too-high value yields decreased performance. 
Especially for the \textit{AcrobotSwingupSparse, HopperHop} and \textit{HopperStand} tasks, correlated exploration yields a significant improvement over the case without correlated exploration. 
However, the results also indicate that there is no global value for $\tauou$ that performs best across every task, which makes tuning this hyperparameter difficult and is only possible if enough compute resources are available. 
Nevertheless, we believe that correlated exploration is a useful feature that \method{} can leverage. 

\begin{figure}[!htbp]\centering
  \includegraphics[width=\textwidth]{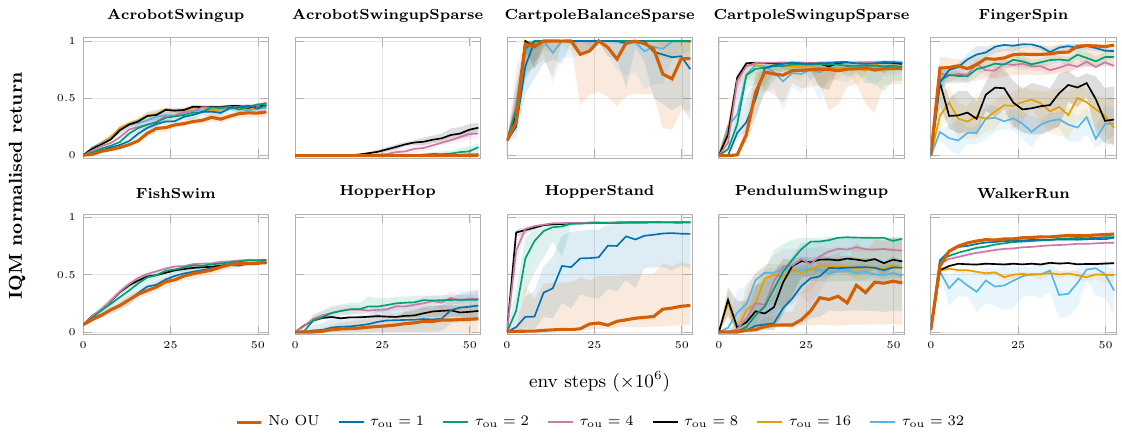}
  \caption{\textbf{Ablation. Correlated Exploration.} On a subset of the MJX Playground tasks, we analyze the effect of correlated exploration by varying $\tauou \in \{1, 2, 4, 8, 16, 32\}$, where we have run every curve on $10$ seeds and report the IQM with $95\%$ stratified bootstrapped confidence intervals \citep{agarwal2021deep}. 
  Correlated exploration can significantly boost \method{}'s performance for the right choice of $\tauou$. For example, the \textit{HopperStand} task is solved, while \method{} without correlated exploration struggles on this task.  
}
  \label{fig:ou-grid}
\end{figure}

\section{Offline-to-online fine-tuning}
\label{app:o2o}

This appendix collects the offline-to-online instantiation of \method{} (\cref{sec:online}): the behavior-constrained setting and the behavior-tilted target (\cref{app:o2o-prelim}), the fine-tuning procedure (\cref{app:o2o-algorithm}), and the protocol, full results, and design ablations of the offline-to-online experiments of \cref{sec:experiments} (\cref{app:exp,app:o2o-results,app:ablations}). Throughout, $\velref$ and $\scoreref$ denote the velocity and score fields of the pretrained (behavior) policy $\pref$, related by the bridge \cref{eq:app-bridge}, and $\cs(t) = \tfrac{\dot\alpha(t)}{\alpha(t)}\beta^2(t) - \dot\beta(t)\beta(t)$ is the bridge coefficient \cref{eq:cs}, assumed positive on $(0,1)$.

\subsection{Further details}
\label{app:o2o-prelim}

In offline-to-online RL a fixed dataset $\mathcal D$ of transitions is available before any interaction. Two objects are pretrained on it: a \emph{behavior flow} $\pref$, fit by flow matching \cref{eq:fm} to the dataset actions (its velocity field $\velref$ is the reference of \cref{sec:online}), and a critic $Q$, fit by an offline method (here pessimistic TD, \cref{alg:o2o}). Fine-tuning then continues both while the dataset grows with online interaction. The difficulty is that an offline critic is trustworthy only near the data: its extrapolation errors off the support of $\pref$ are large and systematically optimistic, so an unconstrained improvement step chases them \citep{peng2019awr,nair2020awac,park2025fql}. The standard remedy penalizes deviation from the behavior policy: at each state, maximize $\E_{x\sim\pi}[Q(\obs,x)] - \lambda\,\kl{\pi(\cdot\mid\obs)}{\pref(\cdot\mid\obs)}$, whose maximizer is the \emph{behavior-tilted} Boltzmann policy $\piplus(x\mid\obs) \propto \pref(x\mid\obs)\,e^{Q(\obs,x)/\lambda}$ of \cref{sec:prelim} \citep{abdolmaleki2018mpo,peng2019awr,nair2020awac}. Its terminal score is
\begin{equation}
  \nabla_x\log\piplus(x\mid\obs) \;=\; \scoreref(x) + \nabla_x Q(\obs,x)/\lambda,
  \qquad
  \scoreref \;:=\; \nabla_x \log\pref(\cdot\mid\obs),
  \label{eq:tilted-score}
\end{equation}
so relative to the pure Boltzmann target of \cref{sec:tfm} exactly one new ingredient enters: the reference's terminal score. For a flow-parameterized reference it is available without a density and without the dataset, by reading the score off the pretrained velocity field through the bridge \cref{eq:app-bridge-inv} near $t=1$, which is \cref{eq:refscore-bridge}. Two mechanisms then keep the fine-tuned policy close to $\pref$, and \cref{alg:o2o} combines them: the \emph{target} itself, which by \cref{app:proofs} has the tilt as a fixed point, and the trust region of \cref{sec:online}, here taken with respect to the pretrained flow, that is, $\mathcal{L}_{\mathrm{TR}}(\theta, \mathrm{ref}) \le \varepsilon_{\mathrm{TR}}$ with $\velref$ in place of $\vel_{\bar\theta}$ in \cref{eq:klflow}, which bounds every step.

\subsection{Pseudocode}
\label{app:o2o-algorithm}

\cref{alg:o2o} summarizes the fine-tuning procedure. It alternates a pessimistic TD update of the critic ensemble with the behavior-tilted \method{} step of \cref{sec:online}: clean actions are generated by the current flow, labeled with the score $g$ of the behavior-tilted target, noised, and regressed onto by \cref{eq:o2o-target}, while the trust region \cref{eq:klflow}, taken with respect to the pretrained flow $\velref$ instead of the previous policy, is enforced through the Lagrangian \cref{eq:lagrangian} and the dual update of \cref{eq:tr_minmax}. We write $\mathcal{L}_{\mathrm{TR}}(\theta, \mathrm{ref})$ for \cref{eq:klflow} with $\velref$ in place of $\vel_{\bar\theta}$, $\bar\theta$ for a frozen copy of the current actor parameters, and $\bar\phi$ for the EMA target parameters of the critic. The implementation differs from the pseudocode in the form of the actor regression, which uses the variance-reduced target of \cref{app:targets}; this and the remaining implementation details are given in \cref{app:exp}.

\begin{algorithm}[H]
\caption{Offline-to-online fine-tuning with an \method{} actor}
\label{alg:o2o}
\begin{algorithmic}[1]
\State \textbf{input:} Pretrained velocity field $\velref$ with terminal law $\pref$, critic ensemble $Q_{\phi_{1:N}}$ with EMA targets $Q_{\bar\phi_{1:N}}$, dataset $\mathcal{D}$, temperature $\lambda$, pessimism $\varrho$, trust-region budget $\varepsilon_{\mathrm{TR}}$ with multiplier $\nu$ and dual step size $\eta_\nu$, schedules $\alpha, \beta$
\State \textbf{initialization:} warm start $\vel_\theta \gets \velref$, set $\nu \gets \nu_0$
\For{each step, first offline on $\mathcal{D}$, then online}
  \State \textbf{interaction (online phase only, no gradients):}
  \Statex \hspace{\algorithmicindent}\hspace{\algorithmicindent} \textbullet\ act $x \sim \pi_\theta(\cdot \mid \obs)$ by integrating \cref{eq:ode}
  \Statex \hspace{\algorithmicindent}\hspace{\algorithmicindent} \textbullet\ add $(\obs, x, r, \obs')$ to $\mathcal{D}$
  \State \textbf{policy evaluation:} pessimistic TD targets on a minibatch of $\mathcal{D}$ with $x' \sim \pi_\theta(\cdot \mid \obs')$ (target gradients stopped)
  \Statex \hspace{\algorithmicindent}\hspace{\algorithmicindent} $\displaystyle y = r + \gamma\,\bigl(\mathrm{mean}_j\, Q_{\bar\phi_j}(\obs', x') - \varrho\,\mathrm{std}_j\, Q_{\bar\phi_j}(\obs', x')\bigr)$
  \State update critic parameters $\phi_{1:N}$ using $\sum_j \bigl(Q_{\phi_j}(\obs, x) - y\bigr)^2$, then the EMA targets $\bar\phi_{1:N}$
  \State \textbf{calculate scores (no gradients):}
  \Statex \hspace{\algorithmicindent}\hspace{\algorithmicindent} \textbullet\ clean actions $X_1 \sim \pi_{\bar\theta}(\cdot \mid \obs)$ by integrating \cref{eq:ode}
  \Statex \hspace{\algorithmicindent}\hspace{\algorithmicindent} \textbullet\ reference score $\scoreref(X_1)$ read off $\velref$ by \cref{eq:refscore-bridge}
  \Statex \hspace{\algorithmicindent}\hspace{\algorithmicindent} \textbullet\ behavior-tilted score $g(X_1) = \scoreref(X_1) + \tfrac{1}{\lambda}\nabla_x Q(\obs, X_1)$
  \State $t \sim \mathcal{U}(0,1)$, $\varepsilon \sim \gauss{\cdot}{0}{I}$, $X_t \gets \alpha(t) X_1 + \beta(t)\varepsilon$
  \State minimize over $\theta$ using \cref{eq:lagrangian} with $\mathcal{L}_{\mathrm{BT-FFM}}$ of \cref{eq:o2o-target}:
  \Statex \hspace{\algorithmicindent}\hspace{\algorithmicindent} $\displaystyle
  \mathbb{E}_{t, X_1, \varepsilon}\Bigl[ \Bigl\| \vel_\theta(X_t, t \mid \obs) - \tfrac{\dot\alpha(t)}{\alpha(t)} X_t - \tfrac{\cs(t)}{\alpha(t)}\, g(X_1) \Bigr\|^2 \Bigr]
  + \nu \bigl( \mathcal{L}_{\mathrm{TR}}(\theta, \mathrm{ref}) - \varepsilon_{\mathrm{TR}} \bigr)$
  \Statex \hspace{\algorithmicindent}\hspace{\algorithmicindent} with $\mathcal{L}_{\mathrm{TR}}(\theta, \mathrm{ref})$ of \cref{eq:klflow} evaluated at the same points $(X_t, t)$
  \State \textbf{dual:} ascend $\nu \gets \max\bigl(0,\, \nu + \eta_\nu\,(\mathcal{L}_{\mathrm{TR}}(\theta, \mathrm{ref}) - \varepsilon_{\mathrm{TR}})\bigr)$ (\cref{eq:tr_minmax})
  \State \textbf{update the frozen actor copy:} $\bar\theta \gets \theta$
\EndFor
\State \textbf{return} fine-tuned velocity field $\vel_\theta$
\end{algorithmic}
\end{algorithm}

\subsection{Experimental details}
\label{app:exp}

\subsubsection{Environments and tasks}
We evaluate on the offline-to-online benchmark of OGBench \citep{park2025ogbench}, comprising ten domains that
split into \emph{locomotion} (goal-reaching navigation) and \emph{manipulation}
(goal-conditioned object rearrangement); \cref{fig:envgrid} shows a representative frame
of each. Every domain defines five \emph{singletask} goals (Task~1--5): the offline
dataset is shared across a domain's tasks (goal-free \texttt{navigate}/\texttt{play}
data) and each task fixes a goal used for the sparse reward and for evaluation.
\Cref{tab:envspecs} lists the state and action dimensionality, episode length,
dataset size, and network width for each domain.

\begin{figure}[t]\centering
  \includegraphics[width=\textwidth]{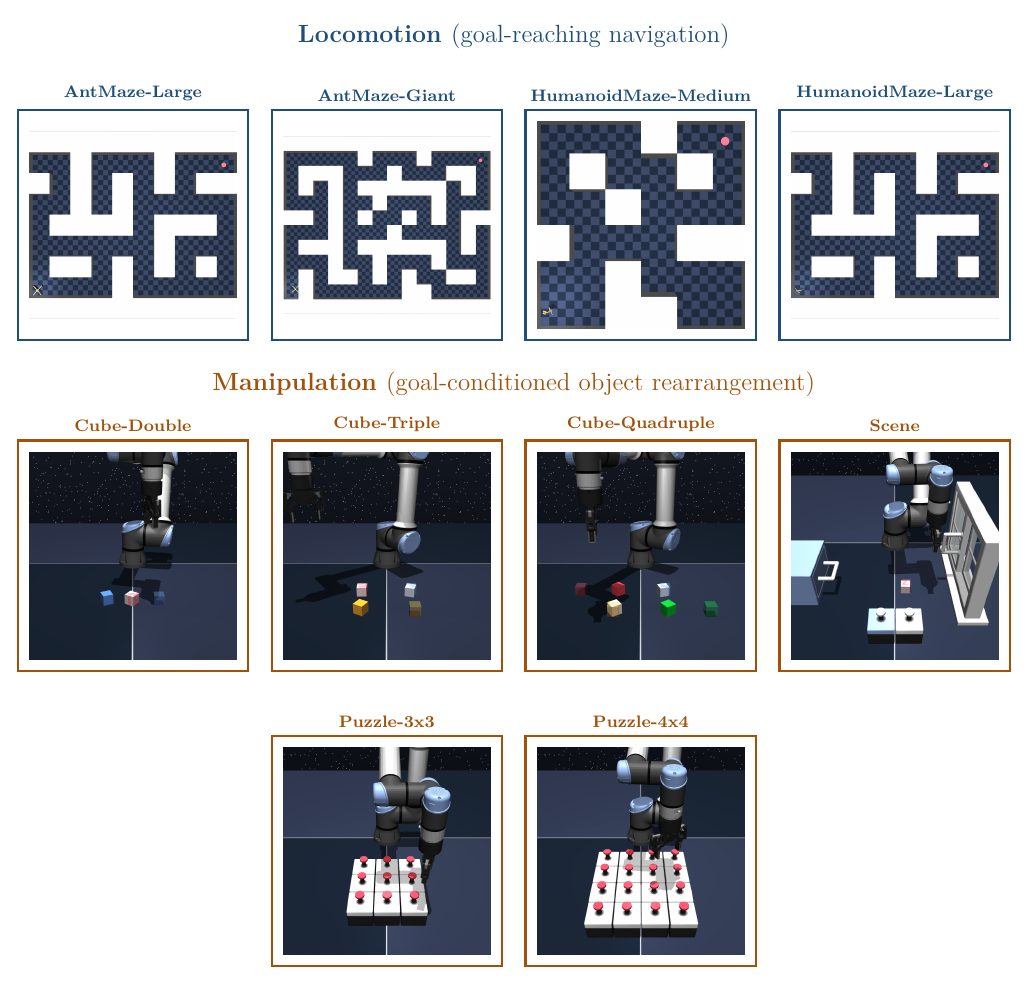}
  \caption{The ten evaluation environments (five tasks each). Top: locomotion; bottom:
  manipulation.}
  \label{fig:envgrid}
\end{figure}

\begin{table}[t]\centering\small
\caption{Environment specifications. Network width and layer norm follow the
dataset-size convention (1M: width $512$; 10M/100M: width $1024$ with actor layer norm),
applied identically to the BC-pretraining and fine-tuning networks.}
\label{tab:envspecs}
\setlength{\tabcolsep}{5pt}
\begin{tabular}{llccccc}
\toprule
Environment & Type & State & Act & Ep.\ len & Dataset & Width \\
\midrule
AntMaze-Large        & locomotion   & 29 & 8  & 1000 & 1M   & 512  \\
AntMaze-Giant        & locomotion   & 29 & 8  & 1000 & 10M  & 1024 \\
HumanoidMaze-Medium  & locomotion   & 69 & 21 & 2000 & 1M   & 512  \\
HumanoidMaze-Large   & locomotion   & 69 & 21 & 2000 & 1M   & 512  \\
Cube-Double          & manipulation & 37 & 5  & 500  & 1M   & 512  \\
Cube-Triple          & manipulation & 46 & 5  & 1000 & 10M  & 1024 \\
Cube-Quadruple       & manipulation & 55 & 5  & 1000 & 100M & 1024 \\
Scene                & manipulation & 40 & 5  & 750  & 1M   & 512  \\
Puzzle-3x3           & manipulation & 55 & 5  & 500  & 1M   & 512  \\
Puzzle-4x4           & manipulation & 83 & 5  & 500  & 10M  & 1024 \\
\bottomrule
\end{tabular}
\end{table}

\subsubsection{Offline-to-online protocol}
We follow the protocol of TRQAM \citep{dong2026trust}. For every (environment, task, seed) we first pretrain a
behavior-cloning flow policy by flow matching for $300$K gradient steps, then run $1$M
offline gradient steps followed by $500$K online environment steps ($1.5$M total); the
online phase continues from the offline actor, critic, optimizer, and dual variables.
Policies are evaluated every $50$K steps over $50$ episodes and we report the binary
task-success rate. We use the eight-seed set $\{10001, 20002, \dots, 80008\}$.

\subsubsection{Shared architecture and hyperparameters}
All methods use four-hidden-layer MLP actors and critics; the critic is an ensemble with
pessimistic backup. The discount and pessimism are $(\gamma,\varrho)=(0.995,0.5)$ except on
the long-horizon HumanoidMaze domains, where $(\gamma,\varrho)=(0.999,0.0)$.
\Cref{tab:shared} lists the shared optimization hyperparameters.

\begin{table}[t]\centering\small
\caption{Shared hyperparameters (all methods).}\label{tab:shared}
\begin{tabular}{ll}
\toprule
Optimizer & Adam, learning rate $3\times10^{-4}$ \\
Batch size & $256$ \\
MLP & 4 hidden layers, width $512$ (1M) / $1024$ (10M--100M) \\
Layer norm & critic: yes; actor: $1024$-width domains only \\
Discount $\gamma$ & $0.995$ \; (HumanoidMaze: $0.999$) \\
Pessimism $\varrho$ & $0.5$ \; (HumanoidMaze: $0.0$) \\
Target EMA rate & $0.005$ \\
Flow integration steps & $10$ \\
Action chunk length (manipulation) & $5$ \\
\bottomrule
\end{tabular}
\end{table}

\subsubsection{\method{} hyperparameters}
\method{} has a single per-domain choice, the regime of \cref{tab:tfmhp}: Regime~A on the
sparse-reward locomotion domains and Regime~B on the chunked manipulation domains, with no
per-task or per-seed tuning. All other settings are shared. The terminal reference score is
always used and read off the pretrained flow by \cref{eq:refscore-bridge} at
$t = 1 - \delta_{\mathrm{ref}}$ with $\delta_{\mathrm{ref}} = 0.1$; the endpoint softening is
$\delta = 0.05$; four interpolant samples $(t, \varepsilon)$ are drawn per clean action; and the
trust-region budget is $\varepsilon_{\mathrm{TR}} = 1$ offline and $8$ online.

\begin{table}[t]\centering\small
\caption{\method{} two-regime recipe: the only per-domain choice.}\label{tab:tfmhp}
\begin{tabular}{lll}
\toprule
 & Regime A (locomotion) & Regime B (manipulation) \\
\midrule
Critic scale $1/\lambda$     & $30$                       & $1$ \\
Fisher normalizer            & fixed ($\bar{\mathcal I}=1$) & adaptive (EMA) \\
Behavior anchor              & trust region on            & trust region off (product tilt only) \\
Action chunking              & no                         & yes ($h=5$) \\
\bottomrule
\end{tabular}
\end{table}

\paragraph{Actor update.} The implementation differs from \cref{alg:o2o} in the actor regression.
Instead of the pure \method{} target of \cref{eq:o2o-target}, it regresses onto the mixed target
$Y_c$ of \cref{eq:velmix} under the variance-optimal schedule $c^\star(t)$ of
\cref{eq:velmix-cstar}, with the behavior-tilted score $g = \scoreref + \nabla_x Q/\lambda$ in
place of $\nabla_x Q/\lambda$ and the Fisher scale $\bar{\mathcal I}$ either fixed or tracked by
an EMA (\cref{tab:tfmhp}). The residual is divided by $\cs(t)$, so that the regression is carried
out in score space (\cref{app:targets}). All offline-to-online experiments use the linear
interpolant $\alpha(t) = t$, $\beta(t) = 1 - t$, for which $\dot\alpha(t)/\alpha(t) = 1/t$ and
$\cs(t) = (1-t)/t$, and the regression line reads
\begin{equation}
  \frac{\vel_\theta(X_t,t) - Y_c}{\cs(t)}
  \;=\; \frac{t\big(\vel_\theta(X_t,t) - X_1\big) + (t - c)\,\varepsilon}{1-t} \;-\; \frac{c}{t}\, g(X_1).
  \label{eq:alg-o2o-line}
\end{equation}
The denominators $1-t$ and $t$ are softened by $\delta$, which replaces a time truncation
(\cref{app:targets}), and the trust region \cref{eq:klflow} is evaluated at the same points
$(X_t, t)$ with the weight $1/\cs(t)$ and the same softening. The clean actions $X_1$ are clipped to
the action box before the critic gradient is taken, the critic target uses the best of $n$ sampled
next actions, and the pretrained flow $\velref$ continues to be trained by flow matching
\cref{eq:fm} on the growing dataset, so that the reference tracks the data.

\paragraph{The two regimes.} The rows of \cref{tab:tfmhp} are two coordinated settings rather than
four independent knobs. In the sparse-reward mazes the critic is smooth and $\nabla_x Q$ is small,
so we amplify it with $1/\lambda = 30$. This only works if the Fisher scale is held fixed at
$\bar{\mathcal I} = 1$: an adaptive normalizer estimates $\|\nabla_x Q/\lambda\|^2$ and divides the
amplification back out. In manipulation the chunked critic gradient is informative, so
$1/\lambda = 1$ suffices, and the adaptive EMA normalizer is preferable because the gradient
magnitudes of an action chunk vary by orders of magnitude across coordinates and over training. The
reference score anchors the target to the pretrained flow in both regimes. Locomotion additionally
keeps the trust region \cref{eq:klflow}, since with sparse rewards the critic is unreliable off the
visited data and a navigation policy can otherwise drift into unseen parts of the maze. In
manipulation the product tilt and chunking already keep the actor on the data, and the trust region
only slows online improvement, so it is switched off. Action chunking with horizon $h = 5$ is part
of the protocol inherited from TRQAM and is applied to all methods on the manipulation domains
(\cref{tab:shared}); it appears in \cref{tab:tfmhp} only because it coincides with the regime split.

\subsubsection{Baselines}
We compare against TRQAM \citep{dong2026trust}, QAM and QAM-E \citep{li2026qam}, FQL and IFQL
\citep{park2025fql}, DSRL \citep{wagenmaker2025dsrl}, and CGQL-L \citep{dong2026trust}, each run in
the identical protocol above.

\subsection{Full results}
\label{app:o2o-results}
\Cref{tab:offline,tab:online} report the offline-phase (at $1$M steps) and
offline-to-online (at $1.5$M steps) success rates, aggregated per domain over the five
tasks and all seeds, with the best method per row in bold. The per-domain aggregated
learning curves are shown in \cref{fig:curves-agg} (main text), and \cref{fig:curves-grid}
gives the full per-(environment, task) grid; in both, the offline ($\le\!1$M) and online
($>\!1$M) phases are shaded. The tables corroborate the discussion of \cref{sec:experiments}:
\method{}'s offline average trails the strongest offline baseline yet its offline-to-online
average is the best overall, with the largest per-environment gains on the hard,
long-horizon tasks and no divergence across seeds.

\begin{table}[t]
\centering\small
\caption{Offline-phase success rate at the end of offline training (1M steps), mean over 5 tasks and seeds per environment. Best per row in bold.}
\label{tab:offline}
\setlength{\tabcolsep}{4pt}
\begin{tabular}{lcccccccc}
\toprule
Environment & \method{} & TRQAM & QAM & QAM-E & IFQL & FQL & DSRL & CGQL-L \\
\midrule
AntMaze-Large & \textbf{0.89} & 0.89 & 0.62 & 0.86 & 0.29 & 0.38 & 0.53 & 0.48 \\
AntMaze-Giant & 0.26 & \textbf{0.41} & 0.29 & 0.06 & 0.12 & 0.02 & 0.01 & 0.07 \\
HumanoidMaze-Medium & \textbf{0.99} & 0.84 & 0.64 & 0.59 & 0.93 & 0.74 & 0.53 & 0.57 \\
HumanoidMaze-Large & \textbf{0.84} & 0.36 & 0.04 & 0.04 & 0.30 & 0.02 & 0.01 & 0.06 \\
Cube-Double & 0.47 & \textbf{0.81} & 0.71 & 0.71 & 0.09 & 0.44 & 0.72 & 0.55 \\
Cube-Triple & 0.25 & \textbf{0.50} & 0.19 & 0.11 & 0.24 & 0.07 & 0.34 & 0.00 \\
Cube-Quadruple & 0.01 & \textbf{0.19} & 0.18 & 0.09 & 0.06 & 0.09 & 0.09 & 0.01 \\
Scene & 0.74 & 0.79 & 0.64 & 0.63 & 0.36 & 0.70 & \textbf{0.80} & 0.58 \\
Puzzle-3x3 & 0.88 & 1.00 & 0.15 & 0.89 & 0.64 & 0.25 & \textbf{1.00} & 0.00 \\
Puzzle-4x4 & 0.95 & \textbf{0.99} & 0.01 & 0.54 & 0.42 & 0.09 & 0.61 & 0.00 \\
\midrule
\textit{Average} & 0.63 & \textbf{0.68} & 0.35 & 0.45 & 0.34 & 0.28 & 0.46 & 0.23 \\
\bottomrule
\end{tabular}
\end{table}

\begin{table}[t]
\centering\small
\caption{Offline-to-online success rate at the end of online fine-tuning (1.5M total steps), mean over 5 tasks and seeds per environment. Best per row in bold.}
\label{tab:online}
\setlength{\tabcolsep}{4pt}
\begin{tabular}{lcccccccc}
\toprule
Environment & \method{} & TRQAM & QAM & QAM-E & IFQL & FQL & DSRL & CGQL-L \\
\midrule
AntMaze-Large & 1.00 & 0.98 & 0.99 & 0.96 & 0.52 & \textbf{1.00} & 0.95 & 0.97 \\
AntMaze-Giant & \textbf{0.98} & 0.97 & 0.74 & 0.97 & 0.44 & 0.84 & 0.90 & 0.61 \\
HumanoidMaze-Medium & \textbf{1.00} & 0.97 & 0.82 & 0.78 & 0.75 & 0.96 & 0.82 & 0.95 \\
HumanoidMaze-Large & \textbf{0.99} & 0.76 & 0.26 & 0.27 & 0.12 & 0.37 & 0.28 & 0.74 \\
Cube-Double & 0.99 & 1.00 & \textbf{1.00} & 1.00 & 0.72 & 1.00 & 0.99 & 0.99 \\
Cube-Triple & 0.99 & \textbf{1.00} & 0.90 & 0.98 & 0.46 & 0.87 & 0.90 & 0.20 \\
Cube-Quadruple & 0.79 & 0.78 & 0.74 & 0.77 & 0.68 & 0.79 & \textbf{0.79} & 0.70 \\
Scene & 0.79 & 0.80 & 0.80 & \textbf{0.80} & 0.46 & 0.80 & \textbf{0.80} & 0.70 \\
Puzzle-3x3 & \textbf{1.00} & \textbf{1.00} & 0.80 & \textbf{1.00} & 0.55 & \textbf{1.00} & \textbf{1.00} & 0.01 \\
Puzzle-4x4 & 0.97 & \textbf{1.00} & 0.80 & \textbf{1.00} & 0.88 & 1.00 & \textbf{1.00} & 0.05 \\
\midrule
\textit{Average} & \textbf{0.95} & 0.93 & 0.78 & 0.85 & 0.56 & 0.86 & 0.84 & 0.59 \\
\bottomrule
\end{tabular}
\end{table}

\begin{figure}[p]\centering
  \resizebox{0.92\textwidth}{!}{\input{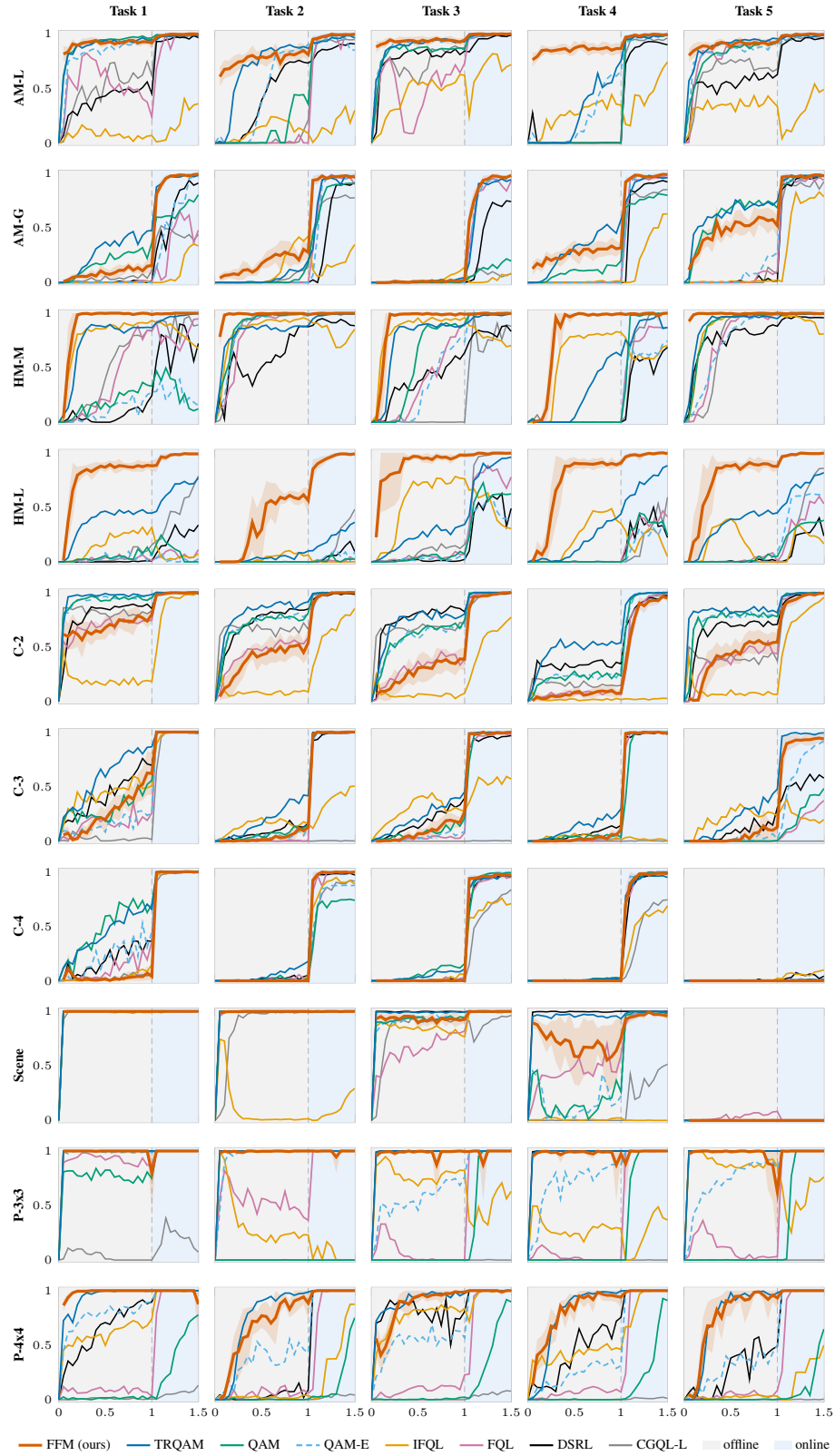}}
  \caption{Offline-to-online success rate for every environment (rows) and task
  (columns), mean over seeds. Offline and online phases are shaded; \method{} (ours) is
  drawn in bold.}
  \label{fig:curves-grid}
\end{figure}

\subsection{Design ablations}
\label{app:ablations}
The recipe of \cref{tab:tfmhp} and the integration budget were fixed once and transferred across
all ten domains. The ablations below justify these choices. Each varies a single setting on one
task per family, \texttt{cube-double} task~3 (manipulation) and \texttt{antmaze-large} task~3
(locomotion), with all other settings at their regime values, and reports the success rate
averaged over the last $150$K steps of the respective phase as mean $\pm$ standard deviation over
three seeds. Unless stated otherwise, the runs are offline only ($1$M steps, reusing the
pretrained flow).

\paragraph{Critic scale and Fisher normalizer.}
The two left panels of \cref{fig:abl-knobs} sweep $1/\lambda \in \{1,2,4,8,30\}$ for the fixed
($\bar{\mathcal I} = 1$) and the adaptive (EMA) normalizer. On manipulation, success falls
monotonically with $1/\lambda$ (from $0.40$ to $0.00$) and the two normalizers coincide, since at
$1/\lambda = 1$ there is no amplification to cancel. On locomotion the curves cross: with the fixed
normalizer success rises with $1/\lambda$ ($0.83 \to 0.95$), with the adaptive one it falls
($0.72 \to 0.14$) because the EMA divides out the gain. The pairs $(1/\lambda{=}30, \text{fixed})$
and $(1/\lambda{=}1, \text{adaptive})$ of \cref{tab:tfmhp} are therefore the two operating points,
and no mixed combination improves on either.
A third variant, \emph{adaptive-shape}, tracks the Fisher normalizer per coordinate but rescales
it to unit mean, so that it keeps the scale invariance across action dimensions without cancelling
the overall $1/\lambda^2$ level. On locomotion it rises with $1/\lambda$ like the fixed normalizer
($0.82 \to 0.93$), and on manipulation all three variants coincide. It would remove the normalizer
row of \cref{tab:tfmhp}, but since it matches rather than improves on the per-domain choice, we keep
the explicit split.

\paragraph{Integration steps.}
The two right panels of \cref{fig:abl-knobs} vary the number of Euler steps over $\{2,5,10,25\}$.
On manipulation two steps are too coarse ($0.31$) and performance saturates from five steps on
($\approx 0.40$); locomotion is insensitive ($0.98$ at two steps, $0.94$ at $25$). The ten steps
inherited from the TRQAM/QAM protocol lie on the plateau for both families.

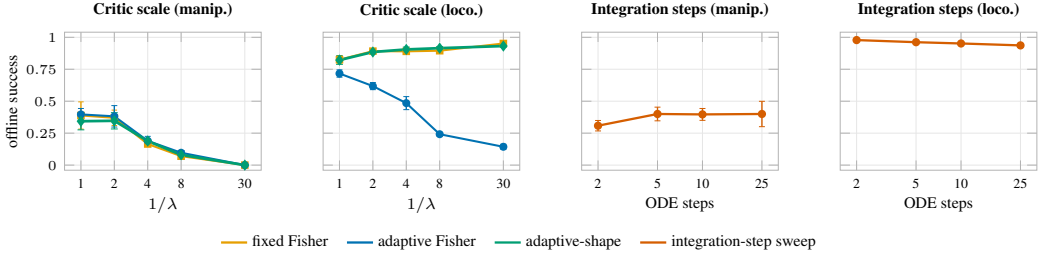
\begin{figure}[t]\centering
  \resizebox{\textwidth}{!}{\definecolor{aFix}{HTML}{E69F00}\definecolor{aAdp}{HTML}{0072B2}%
\definecolor{aShape}{HTML}{009E73}\definecolor{aStep}{HTML}{D55E00}%
\providecommand{\ablleg}[2]{\tikz[baseline=-0.5ex]{\draw[#1,line width=1pt](0,0)--(0.4,0);}~#2\hspace{1.1em}}
\pgfplotsset{
  ablaxis/.style={width=4.6cm,height=3.7cm, ymin=-0.04,ymax=1.04, ytick={0,0.25,0.5,0.75,1},
    tick label style={font=\tiny}, title style={font=\scriptsize\bfseries,yshift=-3pt},
    xlabel style={font=\scriptsize,yshift=3pt}, ylabel style={font=\scriptsize,yshift=-3pt},
    grid=major, grid style={gray!20,line width=0.3pt}, axis line style={gray!60},
    every tick/.style={gray!60}, error bars/error bar style={line width=0.6pt}, },
}
\begin{tikzpicture}
\begin{groupplot}[group style={group size=4 by 1, horizontal sep=0.95cm,
    yticklabels at=edge left}, ablaxis]
\nextgroupplot[title={Critic scale (manip.)}, xmode=log, log basis x=10,
  xtick={1,2,4,8,30}, xticklabels={1,2,4,8,30}, log ticks with fixed point,
  xlabel={$1/\lambda$}, ylabel={offline success}]
  \addplot[aFix,mark=square*,mark size=1.2pt,line width=0.9pt,error bars/.cd,y dir=both,y explicit] table[x=x,y=y,y error=yerr]{tikz_data/abl_cs_c2_fix.dat};
  \addplot[aAdp,mark=*,mark size=1.3pt,line width=0.9pt,error bars/.cd,y dir=both,y explicit] table[x=x,y=y,y error=yerr]{tikz_data/abl_cs_c2_adp.dat};
  \addplot[aShape,mark=diamond*,mark size=1.5pt,line width=1.2pt,error bars/.cd,y dir=both,y explicit] table[x=x,y=y,y error=yerr]{tikz_data/abl_cs_c2_sh.dat};
\nextgroupplot[title={Critic scale (loco.)}, xmode=log, log basis x=10,
  xtick={1,2,4,8,30}, xticklabels={1,2,4,8,30}, log ticks with fixed point, xlabel={$1/\lambda$}]
  \addplot[aFix,mark=square*,mark size=1.2pt,line width=0.9pt,error bars/.cd,y dir=both,y explicit] table[x=x,y=y,y error=yerr]{tikz_data/abl_cs_am_fix.dat};
  \addplot[aAdp,mark=*,mark size=1.3pt,line width=0.9pt,error bars/.cd,y dir=both,y explicit] table[x=x,y=y,y error=yerr]{tikz_data/abl_cs_am_adp.dat};
  \addplot[aShape,mark=diamond*,mark size=1.5pt,line width=1.2pt,error bars/.cd,y dir=both,y explicit] table[x=x,y=y,y error=yerr]{tikz_data/abl_cs_am_sh.dat};
\nextgroupplot[title={Integration steps (manip.)}, xmode=log, log basis x=2,
  xtick={2,5,10,25}, xticklabels={2,5,10,25}, log ticks with fixed point, xlabel={ODE steps}]
  \addplot[aStep,mark=*,mark size=1.3pt,line width=1pt,error bars/.cd,y dir=both,y explicit]
   table[x=x,y=y,y error=yerr]{tikz_data/abl_fs_c2.dat};
\nextgroupplot[title={Integration steps (loco.)}, xmode=log, log basis x=2,
  xtick={2,5,10,25}, xticklabels={2,5,10,25}, log ticks with fixed point, xlabel={ODE steps}]
  \addplot[aStep,mark=*,mark size=1.3pt,line width=1pt,error bars/.cd,y dir=both,y explicit]
   table[x=x,y=y,y error=yerr]{tikz_data/abl_fs_am.dat};
\end{groupplot}
\node[anchor=north, font=\scriptsize, inner sep=0pt] at ($(current bounding box.south)+(0,-0.22cm)$)
  {\ablleg{aFix}{fixed Fisher}\ablleg{aAdp}{adaptive Fisher}\ablleg{aShape}{adaptive-shape}\ablleg{aStep}{integration-step sweep}};
\end{tikzpicture}}
  \caption{\textbf{Critic scale and integration steps} (offline success, mean $\pm$ std over three
  seeds, one manipulation and one locomotion task). \emph{Left two panels:} fixed, adaptive, and
  adaptive-shape Fisher normalizer against $1/\lambda$. \emph{Right two panels:} number of Euler
  steps.}
  \label{fig:abl-knobs}
\end{figure}

\paragraph{Behavior anchor.}
Two mechanisms keep \method{} near the data: the reference score, which makes the \emph{product}
tilt $\pref\, e^{Q/\lambda}$ the target, and the trust region \cref{eq:klflow}. Without the
reference score the target is the \emph{pure} tilt $e^{Q/\lambda}$. \Cref{fig:abl-anchor} crosses
the two (product vs.\ pure tilt, trust region on vs.\ off) in both phases. Offline, removing both
anchors collapses the policy ($0.00$ on manipulation, $0.19$ on locomotion). Either one recovers
it: on \texttt{cube-double} the pure tilt with trust region ($0.58$) even beats the product tilt
($0.41$), and on \texttt{antmaze-large} the product tilt with trust region is best ($0.96$).
Online, at the budget $\varepsilon_{\mathrm{TR}} = 8$ of the recipe, the pure tilt collapses on
\texttt{cube-double} ($0.05$) while the product tilt reaches $0.99$; on locomotion all four
combinations converge to $\approx 1$.

\Cref{fig:abl-klsweep} sweeps the budget $\varepsilon_{\mathrm{TR}}$ on four tasks. The budget the
pure tilt needs is task-dependent in both phases and points in opposite directions. Online,
\texttt{cube-double} needs a tight budget ($1.0$ at $\varepsilon_{\mathrm{TR}} = 1$, $0.05$ at $8$)
and \texttt{puzzle-3x3} a loose one ($0.23$ at $1$, $1.0$ at $8$), so at every fixed budget the
worst task scores $\le 0.42$. Offline, \texttt{cube-double} and \texttt{scene} are best at
$\varepsilon_{\mathrm{TR}} = 1$ ($0.58$ and $1.00$, falling to $0.00$ and $0.02$ at $16$),
\texttt{puzzle-3x3} needs at least $4$ ($0.21$ at $1$, $0.98$ at $4$), \texttt{antmaze-large} is
insensitive, and tightening below $1$ does not help ($0.35$, $0.55$, $0.58$ at
$\varepsilon_{\mathrm{TR}} = 0.25$, $0.5$, $1$ on \texttt{cube-double}). The product tilt, by
contrast, is robust online: it stays $\ge 0.99$ on all four tasks for every
$\varepsilon_{\mathrm{TR}} \ge 4$, so a single budget serves all tasks, which is why the recipe uses
it. Offline it is flat on three tasks but rises from $0.19$ at $\varepsilon_{\mathrm{TR}} = 1$ to
$0.99$ at $8$ on \texttt{puzzle-3x3}; a loose offline budget is equivalent to no trust region
($0.99$ at $16$ vs.\ $1.00$ with the trust region off), which is the setting of the manipulation
regime.

\begin{figure}[t]\centering
  \resizebox{\textwidth}{!}{\definecolor{aKLoff}{HTML}{4C72B0}\definecolor{aKLon}{HTML}{DD8452}%
\providecommand{\ablbox}[2]{\tikz[baseline=-0.5ex]{\fill[#1](0,-0.07)rectangle(0.3,0.09);}~#2\hspace{1.1em}}
\begin{tikzpicture}
\begin{groupplot}[group style={group size=4 by 1, horizontal sep=0.95cm, yticklabels at=edge left},
  width=4.6cm, height=3.7cm, tick label style={font=\tiny}, title style={font=\scriptsize\bfseries,yshift=-3pt},
  ylabel style={font=\scriptsize,yshift=-3pt}, grid=major, grid style={gray!20,line width=0.3pt},
  axis line style={gray!60}, every tick/.style={gray!60}, error bars/error bar style={line width=0.6pt},
  ybar, /pgf/bar width=8pt, enlarge x limits=0.55, ymin=0, ymax=1.08, ytick={0,0.25,0.5,0.75,1},
  symbolic x coords={product,pure}, xtick=data, xticklabels={product\\tilt, pure\\tilt}, xticklabel style={align=center}]
\nextgroupplot[title={Manipulation (offline)}, ylabel={success}]
\addplot[aKLoff,fill=aKLoff,error bars/.cd,y dir=both,y explicit] coordinates {(product,0.412) +- (0,0.100) (pure,0.000) +- (0,0.000)};
\addplot[aKLon,fill=aKLon,error bars/.cd,y dir=both,y explicit] coordinates {(product,0.368) +- (0,0.082) (pure,0.582) +- (0,0.074)};
\nextgroupplot[title={Manipulation (online)}]
\addplot[aKLoff,fill=aKLoff,error bars/.cd,y dir=both,y explicit] coordinates {(product,0.992) +- (0,0.008) (pure,0.025) +- (0,0.035)};
\addplot[aKLon,fill=aKLon,error bars/.cd,y dir=both,y explicit] coordinates {(product,0.993) +- (0,0.006) (pure,0.052) +- (0,0.006)};
\nextgroupplot[title={Locomotion (offline)}]
\addplot[aKLoff,fill=aKLoff,error bars/.cd,y dir=both,y explicit] coordinates {(product,0.812) +- (0,0.030) (pure,0.185) +- (0,0.166)};
\addplot[aKLon,fill=aKLon,error bars/.cd,y dir=both,y explicit] coordinates {(product,0.965) +- (0,0.016) (pure,0.947) +- (0,0.010)};
\nextgroupplot[title={Locomotion (online)}]
\addplot[aKLoff,fill=aKLoff,error bars/.cd,y dir=both,y explicit] coordinates {(product,0.998) +- (0,0.002) (pure,0.943) +- (0,0.023)};
\addplot[aKLon,fill=aKLon,error bars/.cd,y dir=both,y explicit] coordinates {(product,1.000) +- (0,0.000) (pure,0.995) +- (0,0.004)};
\end{groupplot}
\node[anchor=north, font=\scriptsize, inner sep=0pt] at ($(current bounding box.south)+(0,-0.22cm)$)
  {\ablbox{aKLoff}{trust region off}\ablbox{aKLon}{trust region on}};
\end{tikzpicture}}
  \caption{\textbf{Behavior anchor.} Product tilt $\pref\,e^{Q/\lambda}$ vs.\ pure tilt
  $e^{Q/\lambda}$, with the trust region on or off, offline and online (mean $\pm$ std over three
  seeds; the online runs use $\varepsilon_{\mathrm{TR}} = 8$).}
  \label{fig:abl-anchor}
\end{figure}
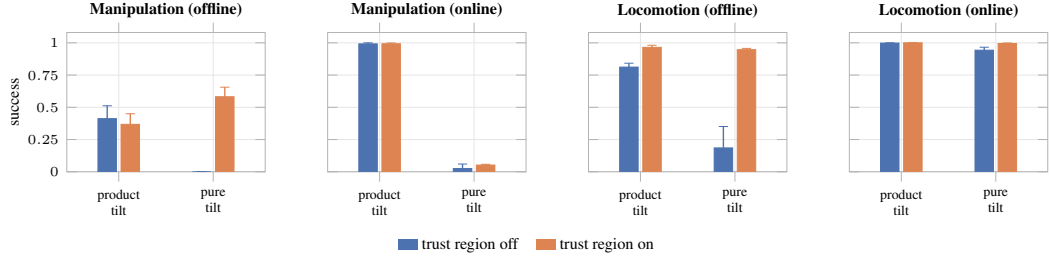

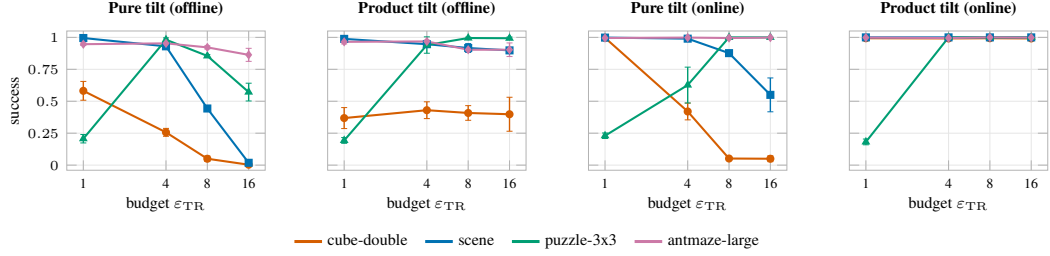
\begin{figure}[t]\centering
  \resizebox{\textwidth}{!}{\definecolor{kC2}{HTML}{D55E00}\definecolor{kSC}{HTML}{0072B2}%
\definecolor{kPZ}{HTML}{009E73}\definecolor{kAM}{HTML}{CC79A7}%
\providecommand{\ablleg}[2]{\tikz[baseline=-0.5ex]{\draw[#1,line width=1.3pt](0,0)--(0.42,0);}~#2\hspace{1.0em}}
\begin{tikzpicture}
\begin{groupplot}[group style={group size=4 by 1, horizontal sep=0.95cm, yticklabels at=edge left},
  width=4.6cm, height=3.7cm, tick label style={font=\tiny}, title style={font=\scriptsize\bfseries,yshift=-3pt},
  xlabel style={font=\scriptsize,yshift=3pt}, ylabel style={font=\scriptsize,yshift=-3pt},
  grid=major, grid style={gray!20,line width=0.3pt}, axis line style={gray!60},
  every tick/.style={gray!60}, error bars/error bar style={line width=0.6pt},
  xmode=log, log basis x=2, xtick={1,4,8,16}, xticklabels={1,4,8,16}, log ticks with fixed point,
  xlabel={budget $\varepsilon_{\mathrm{TR}}$}, ymin=-0.04, ymax=1.05, ytick={0,0.25,0.5,0.75,1}]
\nextgroupplot[title={Pure tilt (offline)}, ylabel={success}]
  \addplot[kC2,mark=*,mark size=1.3pt,line width=0.9pt,error bars/.cd,y dir=both,y explicit] table[x=x,y=y,y error=yerr]{tikz_data/abl_klswoff_c2_pure.dat};
  \addplot[kSC,mark=square*,mark size=1.3pt,line width=0.9pt,error bars/.cd,y dir=both,y explicit] table[x=x,y=y,y error=yerr]{tikz_data/abl_klswoff_sc_pure.dat};
  \addplot[kPZ,mark=triangle*,mark size=1.3pt,line width=0.9pt,error bars/.cd,y dir=both,y explicit] table[x=x,y=y,y error=yerr]{tikz_data/abl_klswoff_pz_pure.dat};
  \addplot[kAM,mark=diamond*,mark size=1.3pt,line width=0.9pt,error bars/.cd,y dir=both,y explicit] table[x=x,y=y,y error=yerr]{tikz_data/abl_klswoff_am_pure.dat};
\nextgroupplot[title={Product tilt (offline)}]
  \addplot[kC2,mark=*,mark size=1.3pt,line width=0.9pt,error bars/.cd,y dir=both,y explicit] table[x=x,y=y,y error=yerr]{tikz_data/abl_klswoff_c2_prod.dat};
  \addplot[kSC,mark=square*,mark size=1.3pt,line width=0.9pt,error bars/.cd,y dir=both,y explicit] table[x=x,y=y,y error=yerr]{tikz_data/abl_klswoff_sc_prod.dat};
  \addplot[kPZ,mark=triangle*,mark size=1.3pt,line width=0.9pt,error bars/.cd,y dir=both,y explicit] table[x=x,y=y,y error=yerr]{tikz_data/abl_klswoff_pz_prod.dat};
  \addplot[kAM,mark=diamond*,mark size=1.3pt,line width=0.9pt,error bars/.cd,y dir=both,y explicit] table[x=x,y=y,y error=yerr]{tikz_data/abl_klswoff_am_prod.dat};
\nextgroupplot[title={Pure tilt (online)}]
  \addplot[kC2,mark=*,mark size=1.3pt,line width=0.9pt,error bars/.cd,y dir=both,y explicit] table[x=x,y=y,y error=yerr]{tikz_data/abl_klsw_c2_pure.dat};
  \addplot[kSC,mark=square*,mark size=1.3pt,line width=0.9pt,error bars/.cd,y dir=both,y explicit] table[x=x,y=y,y error=yerr]{tikz_data/abl_klsw_sc_pure.dat};
  \addplot[kPZ,mark=triangle*,mark size=1.3pt,line width=0.9pt,error bars/.cd,y dir=both,y explicit] table[x=x,y=y,y error=yerr]{tikz_data/abl_klsw_pz_pure.dat};
  \addplot[kAM,mark=diamond*,mark size=1.3pt,line width=0.9pt,error bars/.cd,y dir=both,y explicit] table[x=x,y=y,y error=yerr]{tikz_data/abl_klsw_am_pure.dat};
\nextgroupplot[title={Product tilt (online)}]
  \addplot[kC2,mark=*,mark size=1.3pt,line width=0.9pt,error bars/.cd,y dir=both,y explicit] table[x=x,y=y,y error=yerr]{tikz_data/abl_klsw_c2_prod.dat};
  \addplot[kSC,mark=square*,mark size=1.3pt,line width=0.9pt,error bars/.cd,y dir=both,y explicit] table[x=x,y=y,y error=yerr]{tikz_data/abl_klsw_sc_prod.dat};
  \addplot[kPZ,mark=triangle*,mark size=1.3pt,line width=0.9pt,error bars/.cd,y dir=both,y explicit] table[x=x,y=y,y error=yerr]{tikz_data/abl_klsw_pz_prod.dat};
  \addplot[kAM,mark=diamond*,mark size=1.3pt,line width=0.9pt,error bars/.cd,y dir=both,y explicit] table[x=x,y=y,y error=yerr]{tikz_data/abl_klsw_am_prod.dat};
\end{groupplot}
\node[anchor=north, font=\scriptsize, inner sep=0pt] at ($(current bounding box.south)+(0,-0.22cm)$)
  {\ablleg{kC2}{cube-double}\ablleg{kSC}{scene}\ablleg{kPZ}{puzzle-3x3}\ablleg{kAM}{antmaze-large}};
\end{tikzpicture}}
  \caption{\textbf{Trust-region budget sweep.} Success against $\varepsilon_{\mathrm{TR}}$ on four
  tasks (mean $\pm$ std over three seeds), offline (left two panels) and online (right two
  panels), for the pure tilt (first and third panel) and the product tilt (second and fourth
  panel). With the offline trust region off, the product tilt scores $0.41$, $0.94$, $1.00$, and
  $0.81$ on \texttt{cube-double}, \texttt{scene}, \texttt{puzzle-3x3}, and \texttt{antmaze-large}.}
  \label{fig:abl-klsweep}
\end{figure}

\end{document}